\documentclass{article}

\usepackage[dvipsnames]{xcolor}
\usepackage{sectsty}

     \PassOptionsToPackage{numbers, compress}{natbib}

 \usepackage[final]{neurips_2019}

\usepackage[utf8]{inputenc} 
\usepackage[T1]{fontenc}    
\usepackage{hyperref}       
\usepackage{url}            
\usepackage{booktabs}       
\usepackage{amsfonts}       
\usepackage{nicefrac}       
\usepackage{microtype}      

\usepackage[backgroundcolor=White,textwidth=0.8in]{todonotes}

\usepackage[ruled,noend]{algorithm2e}

\usepackage{appendix}
\usepackage{amsthm}
\usepackage{amsfonts,amsmath,amssymb}
\usepackage{mathtools} 
\usepackage[noabbrev]{cleveref}
\allowdisplaybreaks

\newtheorem{theorem}{Theorem}
\newtheorem*{theorem*}{Theorem}
\newtheorem{corollary}{Corollary}
\newtheorem{lemma}{Lemma}
\newtheorem{proposition}{Proposition}

\DeclareMathOperator*{\argmax}{arg\,max\,}
\DeclareMathOperator*{\argmin}{arg\,min\,}
\DeclarePairedDelimiter{\ceil}{\lceil}{\rceil}
\mathchardef\mhyphen="2D

\newcommand{\cA}{\mathcal{A}}

\newcommand{\cC}{\mathcal{C}}

\newcommand{\length}{l}

\newcommand{\tV}{\widetilde{V}}

\newcommand{\LinUCB}{{\tt LinUCB}}
\newcommand{\DLinUCB}{{\tt D\mhyphen LinUCB}}
\newcommand{\dLinUCB}{{\tt dLinUCB}}
\newcommand{\SWLinUCB}{{\tt SW\mhyphen LinUCB}}

\newcommand{\OmniLinUCB}{{\tt LinUCB\mhyphen OR}}

\usepackage{dsfont}

\usepackage{verbatim} 

\title{Weighted Linear Bandits for Non-Stationary Environments}

 \author{%
 Yoan Russac
    \\
CNRS, Inria, ENS, Université PSL \\
  \texttt{yoan.russac@ens.fr} \\
  \And
  Claire Vernade \\
  Deepmind \\
  \texttt{vernade@google.com} \\
  \And 
  Olivier Cappé \\
  CNRS, Inria, ENS, Université PSL \\
  \texttt{olivier.cappe@cnrs.fr} \\
}

\begin{document}

\maketitle

\begin{abstract}
  We consider a  stochastic linear bandit model in which the available actions
  correspond to arbitrary context vectors whose associated rewards
  follow a \emph{non-stationary} linear regression model.
  In this setting, the unknown regression parameter is allowed to vary in time.  To address this problem, we propose
  $\DLinUCB$, a novel optimistic algorithm based on discounted linear regression, where exponential weights are used to smoothly forget
  the past.
  This involves  studying the deviations of the sequential weighted least-squares estimator under generic assumptions.
  As a by-product, we obtain novel deviation results that can be used  beyond non-stationary environments.
   We provide theoretical guarantees on the behavior of
  $\DLinUCB$ in both slowly-varying and abruptly-changing
  environments. We obtain an upper bound on the
  dynamic regret that is of order $d^{2/3} B_T^{1/3}T^{2/3}$, where $B_T$
  is a measure of non-stationarity ($d$ and $T$ being, respectively, dimension and horizon). This rate is known to be optimal.
  We
  also illustrate the empirical performance of  $\DLinUCB$
  and compare it with recently proposed alternatives in
  simulated environments.
\end{abstract}

\section{Introduction}

Multi-armed bandits offer a class of models to address sequential
learning tasks that involve exploration-exploitation trade-offs. In
this work we are interested in structured bandit models,
known as stochastic linear bandits, in which
linear regression is used to predict rewards \cite{abbasi2011improved, auer2002using, li2010contextual}.

A typical application of bandit algorithms based on the linear model is online
recommendation where actions are items to be, for instance, efficiently
arranged on personalized web pages to maximize some conversion rate.
However, it is unlikely that customers' preferences remain stable
and the collected data becomes progressively obsolete
as the interest for the items evolve.
Hence, it is essential to design adaptive bandit agents
rather than restarting the learning from scratch on a
regular basis. In this work, we consider the use of weighted
least-squares as an efficient method
to progressively forget past interactions. Thus, we address sequential
learning problems
 in which the parameter of the linear
bandit is evolving with time.

Our first contribution consists in extending existing deviation inequalities
to sequential weighted least-squares. Our result applies to a large variety
of bandit problems and is of independent interest. In
particular, it extends the recent analysis of heteroscedastic
environments by \cite{kirschner2018information}. It can also be useful
to deal with class imbalance situations, or, as we focus on here, in
non-stationary environments.

As a second major contribution, we apply our results to
propose  $\DLinUCB$, an adaptive linear bandit algorithm based on
 carefully designed exponential weights.
$\DLinUCB$ can be implemented fully
recursively ---without requiring the storage of past actions--- with a
numerical complexity that is comparable to that of $\LinUCB$. To
characterize the performance of the algorithm, we provide a unified regret analysis for abruptly-changing or
slowly-varying environments.

The setting and notations are presented below and we state our main deviation result in Section~\ref{sec:confidence}.
Section~\ref{sec:non-stationary} is dedicated to non-stationary linear bandits:
we describe our algorithms and provide regret upper bounds in abruptly-changing and slowly-varying environments.
We complete this theoretical study with a
set of experiments in Section~\ref{sec:experiments}.


\subsection{Model and Notations}
The setting we consider in this paper is a non-stationary variant of the
stochastic linear bandit problem considered in \cite{abbasi2011improved,li2010contextual}, where, at each round $t\geq 1$, the learner
\begin{itemize}
    \item receives a finite set of feasible actions $\mathcal{A}_t \subset \mathds{R}^d$;
    \item chooses an action $A_t \in \mathcal{A}_t$ and receives a reward $X_t$ such that
    \begin{equation}
        \label{eq:reward_generation}
        X_t = \left\langle A_t, \theta_t^\star\right\rangle + \eta_t,
    \end{equation}
    where $\theta_t^\star \in \mathds{R}^d$ is an unknown parameter and $\eta_t$ is, conditionally on the past, a $\sigma-$subgaussian random noise.
\end{itemize}

The action set $\mathcal{A}_t$ may be arbitrary but its components are assumed to be
bounded, in the sense that $\| a \|_2 \leq L$, $\forall a \in
\cA_t$. The time-varying parameter is also assumed to be bounded:
$ \forall t, \| \theta_t^\star\|_2  \leq S$. We further assume that
$\vert \left\langle a,\theta_t^\star\right\rangle \vert \leq 1$,
$\forall t, \forall a \in \cA_t,$ (obviously, this could be guaranteed by assuming
that $L=S=1$, but we indicate the dependence in $L$ and $S$ in order
to facilitate the interpretation of some results). For a positive definite matrix $M $ and a vector $x$, we denote by $\lVert x \rVert_M$ the norm $\sqrt{x^{\top} M x}$.

The goal of the learner is to minimize the expected
\emph{dynamic regret} defined as
\begin{equation}
    R(T) = \mathds{E}\left[\sum_{t=1}^T \max_{a\in \cA_t} \left\langle a, \theta_t^\star \right\rangle - X_t \right] = \sum_{t=1}^T \max_{a\in \cA_t} \left\langle a - A_t, \theta_t^\star \right\rangle.
    \label{eq:def_regret}
\end{equation}
Even in the stationary case ---i.e., when
$\theta_t^\star = \theta^\star$---, there is, in general, no single
fixed best action in this model.

When making stronger structural assumption
on $\mathcal{A}_t$, one recovers specific instances that have also
been studied in the literature. In particular, the canonical basis of $\mathds{R}^d$,
$\mathcal{A}_t = \{e_1,\dots, e_d\}$,  yields the familiar ---non contextual--- multi-armed bandit
model~\cite{lattimore2019bandit}. Another variant, studied
by~\cite{goldenshluger2013} and others, is obtained when
$\mathcal{A}_t = \{e_1 \otimes a_t ,\dots, e_k \otimes a_t\}$, where $\otimes$ denotes the Kronecker product and $a_t$
is a time-varying context vector shared by the $k$ actions.


\subsection{Related Work}

There is an important literature on online learning in changing environments. For the sake of conciseness, we restrict the discussion to works that consider specifically the stochastic linear bandit model in~\eqref{eq:reward_generation}, including its restriction to the simpler (non-stationnary) multi-armed bandit model. Note that there is also a rich line of works that consider possibly non-linear contextual models in the case where one can make probabilistic assumptions on the contexts \cite{chen2019new,luo2017efficient}.

Controlling the regret with respect to the non-stationary optimal action defined in~\eqref{eq:def_regret} depends on the assumptions that are made on the time-variations of $\theta^\star_t$. A generic way of quantifying them is through a \emph{variation bound} $B_T = \sum_{s=1}^{T-1} \lVert \theta^{\star}_s - \theta^{\star}_{s+1} \rVert_{2}$ \cite{besbes2014stochastic,besbes2018optimal,cheung2018learning},
similar to the penalty used in the group fused Lasso \cite{bleakley2011group}.
The main advantage of using the variation budget is that is includes both \emph{slowly-varying} and \emph{abruptly-changing} environments.
For the $K-$armed bandits with known $B_T$, \cite{besbes2014stochastic,besbes2015non,besbes2018optimal} achieve the tight dynamic regret bound of $O(K^{1/3}B_T^{1/3}T^{2/3})$.
For linear bandits, \cite{cheung2018learning,cheung2019hedging} propose an algorithm based on the use of a sliding-window and provide a $O(d^{2/3} B_T^{1/3}T^{2/3})$ dynamic regret bound; since this contribution is close to ours, we discuss it further in Section~\ref{sec:analysis}.

A more specific non-stationary setting arises when the number of changes
in the parameter is bounded by $\Gamma_T$, as in traditional
change-point models.
The problem is usually referred to as \emph{switching bandits} or \emph{abruptly-changing} environments.
It is, for instance, the setting considered in the work by \citet{garivier2011upper}, who analyzed the dynamic regret of UCB strategies based on either a sliding-window or exponential discounting.
For both policies, they prove upper bounds on the regret in
$O(\sqrt{\Gamma_T T})$ when $\Gamma_T$ is known.
They also provide a lower bound in a
specific non-stationary setting, showing that $R(T)=\Omega(\sqrt{T})$.
The algorithm ideas can be traced back to \cite{kocsis2006discounted}.
\cite{wei2018abruptly} shows that an horizon-independent version of the sliding window algorithm can also be analyzed in a slowly-varying setting. \cite{keskin2017pricing} analyze windowing and discounting approaches to address dynamic pricing guided by a (time-varying) linear regression model. Discount factors have also been used with Thomson sampling in dynamic environments as in \cite{gupta2011thompson,raj2017taming}.

In abruptly-changing environments, the alternative approach relies on change-point detection \cite{auer2018adaptively,besson2019generalized, cao2018nearly,
wu2018learning, yu2009piecewise}.
A bound
on the regret in $O((\frac{1}{\epsilon^2}+\frac{1}{\Delta})\log(T))$
is proven by \cite{yu2009piecewise}, where $\epsilon$ is the smallest
gap that can be detected by the algorithm, which had to be given as prior knowledge. \cite{cao2018nearly}
proves a minimax bound in $O(\sqrt{\Gamma_TKT})$ if $\Gamma_T$ is known. \cite{besson2019generalized} 
achieves a rate of $O(\sqrt{\Gamma_TKT})$ without any prior knowledge of the gaps or $\Gamma_T$.  In the contextual case,
\cite{wu2018learning} builds on the same idea: they use a pool of $\LinUCB$ learners called \textit{slave models}
 as experts and they add a new model when no existing slave is able to give good prediction,
 that is, when a change is detected.
A limitation however of such an approach is that it can not adapt to
 some slowly-varying environments, as will be illustrated in Section \ref{sec:experiments}. From a practical viewpoint, 
 the methods based either on sliding window or change-point detection require the storage
of past actions whereas those based on discount factors can be
implemented fully recursively.

Finally, non-stationarity may also arise in more specific scenarios connected, for instance, to the decaying attention of the users,
as investigated in
\cite{levine2017rotting, mintz2017non, seznec2018rotting}. In the
following, we consider the general case where the
parameters satisfy the variation bound, i.e., $\sum_{t=1}^{T-1} \lVert \theta^{\star}_t-\theta^{\star}_{t+1} \rVert_2 \leq B_T$ 
and we propose an algorithm based on discounted linear regression.

\section{Confidence Bounds for Weighted Linear Bandits}
In this section, we consider the concentration of the weighted
regularized least-squares estimator, when used with general weights
and regularization parameters. To the best of our knowledge there is
no such results in the literature for sequential learning ---i.e.,
when the current regressor may depend on the random outcomes observed
in the past. The particular case considered in Lemma 5 of
\cite{kirschner2018information} (heteroscedastic noise with optimal
weights) stays very close to the unweighted case and we show below how
to extend this result. We believe that this new bound is of interest beyond the
specific model considered in this paper. For the sake of clarity, we
first focus on the case of regression models with fixed parameter,
where $\theta_t^\star = \theta^\star$, for all $t$.

First consider a deterministic sequence of regularization parameters
$(\lambda_t)_{t \geq 1}$. The reason why these should be non-constant
for weighted least-squares will appear clearly in
Section~\ref{sec:non-stationary}. Next, define by
$\mathcal{F}_t = \sigma(X_1,\dots,X_t)$ the filtration associated with
the random observations. We assume that both the actions $A_t$ and
positive weights $w_t$ are predictable, that is, they are
$\mathcal{F}_{t-1}$ measurable.

Defining by
$$
\hat{\theta}_t = \argmin_{\theta \in \mathbb{R}^d}\left( \sum_{s=1}^t
  w_s (X_s - \langle A_s, \theta \rangle)^2 + \lambda_t
  \Vert \theta \Vert_2^2 \right)
 $$
the regularized weighted
least-squares estimator of $\theta^\star$ at time $t$, one has
\begin{equation}
  \label{eq:thetahatw}
  \hat{\theta}_t = V_t^{-1} \sum_{s=1}^{t} w_s A_s X_s \quad \text{where} \quad V_t = \sum_{s=1}^{t} w_s A_s A_s^\top + \lambda_t I_d ,
\end{equation}
and $I_d$ denotes the $d$-dimensional identity matrix. We further consider an arbitrary sequence of positive 
 parameters $(\mu_t)_{t\geq1}$ and define the matrix
\begin{equation}
  \label{eq:other_qty}
  \tV_t = \sum_{s=1}^{t} w^2_s A_s A_s^\top + \mu_t I_d.
\end{equation}
$\tV$ is strongly connected to the variance of the estimator $\hat{\theta}_t$, which involves the squares of the weights $(w_s^2)_{s \geq 1}$.
For the time being, $\mu_t$ is arbitrary and will be set as a function of $\lambda_t$ in order to optimize the deviation inequality.

We then have the following maximal deviation inequality. 

\begin{theorem}
\label{theorem_deviation_weighted_sequential}
For any $\mathcal{F}_t$-predictable sequences of actions $(A_t)_{t\geq 1}$ and positive weights $(w_t)_{t\geq 1}$ and  for all  $\delta > 0$,
$$
\mathbb{P}\left(\forall t,\lVert \hat{\theta}_t - \theta^{\star} \rVert_{V_t \tV_t^{-1} V_t} \leq  \frac{\lambda_t}{\sqrt{\mu_t}}S +
 \sigma \sqrt{ 2\log(1/\delta) + d\log\left(1+  \frac{L^2 \sum_{s=1}^t w_s^2}{d \mu_t}\right)} \right) \geq 1-\delta.
$$
\end{theorem}

The proof of this theorem is deferred to the appendix and combines
an argument using the method of mixtures and the use of a proper
stopping time. The standard result used for
least-squares~\cite[Chapter 20]{lattimore2019bandit} is recovered by
taking $\mu_t=\lambda_t$ and $w_t=1$ (note that
$\widetilde{V}_t$ is then equal to $V_t$). When the weights are not
equal to 1, the appearance of the matrix $\widetilde{V}_t$ is a consequence 
of the fact that the variance terms are proportional to the
squared weights $w_t^2$, while the least-squares estimator itself is
defined with the weights $w_t$. In the weighted case, the matrix $V_t \tV_t^{-1} V_t$
must be used to define the confidence ellipsoid.

An important property of the least-squares 
estimator is to be scale-invariant, in the sense that
multiplying all weights $(w_s)_{1\leq s\leq t}$ and the regularization
parameter $\lambda_t$ by a constant leaves the estimator
$\hat\theta_t$ unchanged. In
Theorem~\ref{theorem_deviation_weighted_sequential}, the only choice
of sequence $(\mu_t)_{t\geq 1}$ that is compatible with this scale-invariance property
is to take $\mu_t$ proportional to $\lambda_t^2$: then the
matrix $V_t \widetilde{V}_t^{-1} V_t$ becomes scale-invariant 
(\textit{i.e.} unchanged by the transformation $w_s \mapsto \alpha w_s$)
and so
does the upper bound of
$\lVert \hat{\theta}_t - \theta^{\star} \rVert_{V_t \tV_t^{-1} V_t}$
in Theorem~\ref{theorem_deviation_weighted_sequential}. In the
following, we will stick to this choice, while particularizing the
choice of the weights $w_t$ to allow for non-stationary models.

It is possible to extend this result to heteroscedastic noise, when
$\eta_t$ is $\sigma_t$ sub-Gaussian and $\sigma_t$ is
$\mathcal{F}_{t-1}$ measurable, by defining $\widetilde{V}_t$ as
$\sum_{s=1}^t w_s^2 \sigma_s^2 A_s A_s^{\top} + \mu_t I_d$. In the
next section, we will also use an extension of
Theorem~\ref{theorem_deviation_weighted_sequential} to the non-stationary model
presented in ~(\ref{eq:reward_generation}) . In this case,
Theorem~\ref{theorem_deviation_weighted_sequential} holds with
$\theta^\star$ replaced by
$V_{t}^{-1} \bigl(\sum_{s=1}^{t} w_s A_s A_s^{\top} \theta_s^\star +
  \lambda_t \theta^\star_r \bigr)$, where $r$ is an arbitrary time
index (proposition \ref{proposition:bar_theta} in Appendix). The fact that $r$ can be chosen freely is a consequence of the
assumption that the sequence of L2-norms of the parameters $(\theta^{\star}_t)_{t\geq 1}$ is
bounded by $S$.


\label{sec:confidence}

\section{Application to Non-stationary Linear Bandits}
\label{sec:non-stationary}
In this section, we consider the non-stationary model defined in~(\ref{eq:reward_generation})
and propose a bandit algorithm in Section~\ref{sec:algorithm},
called Discounted Linear Upper Confidence Bound ($\DLinUCB$),
that relies on weighted least-squares to adapt to changes in the parameters
$\theta_t^\star$. Analyzing the performance of $\DLinUCB$ in Section~\ref{sec:analysis}, 
we show that it achieves reliable performance both for abruptly changing or slowly drifting parameters.

\subsection{The $\DLinUCB$ Algorithm}
\label{sec:algorithm}

Being adaptive to parameter changes indeed implies to reduce
the influence of observations that are far back in the past, which
suggests using weights $w_t$ that increase with time. In doing so,
there are two important caveats to consider. First, this can only be
effective if the sequence of weights is growing sufficiently fast
(see the analysis in the next section). We thus consider exponentially
increasing weights of the form $w_t = \gamma^{-t}$, where $0<\gamma < 1$
is the discount factor.

 Next, due to the absence of assumptions on
the action sets $\cA_t$, the regularization is instrumental in
obtaining guarantees of the form given in
Theorem~\ref{theorem_deviation_weighted_sequential}. In fact,
if $w_t = \gamma^{-t}$ while $\lambda_t$ does not increase sufficiently fast,
then the term
$\log\bigl(1+  (L^2 \sum_{s=1}^t w_s^2)/(d \mu_t) \bigr)$ will eventually
dominate  the radius of the confidence region since we
choose $\mu_t$ proportional to $\lambda_t^2$.
This occurs because there is no guarantee that the algorithm will
persistently select actions $A_t$ that span the entire space.
 With this in mind, we consider an increasing
 regularization factor of the form
$\lambda_t = \gamma^{-t} \lambda$, where $\lambda >0$
is a hyperparameter.

Note that due to the scale-invariance property of the weighted
least-square estimator, we can equivalently consider that at time $t$,
we are given \emph{time-dependent} weights
$w_{t,s} = \gamma^{t-s}$, for $1\leq s \leq t$ and that
$\hat{\theta}_t$ is defined as
$$
\argmin_{\theta \in \mathbb{R}^d}\bigl( \sum_{s=1}^t \gamma^{t-s}
(X_s - \langle A_s, \theta \rangle)^2 + \lambda/2 \Vert \theta
\Vert_2^2 \bigr).$$
 For numerical stability reasons, this form is
preferable and is used in the statement of
Algorithm~\ref{alg:DLinUCB}. In the analysis of
Section~\ref{sec:analysis} however we revert to the standard form
of the weights, which is required to apply the concentration result of
Section~\ref{theorem_deviation_weighted_sequential}. We are now ready
to describe $\DLinUCB$ in Algorithm~\ref{alg:DLinUCB}.

\begin{algorithm}[hbtp]
  \label{alg:DLinUCB}
\SetKwInput{KwData}{Input}
\KwData{Probability $\delta$, subgaussianity constant $\sigma$, dimension $d$, regularization $\lambda$, 
upper bound for actions $L$, upper bound for parameters $S$, discount factor $\gamma$.}
\SetKwInput{KwResult}{Initialization}
\KwResult{$b = 0_{\mathbb{R}^d}$, $V = \lambda I_d$, $\widetilde{V} = \lambda I_d$, $\hat{\theta} = 0_{\mathbb{R}^d}$}
\For{$t \geq 1$}{
Receive $\mathcal{A}_t$, compute $\beta_{t-1} = \sqrt{\lambda}S + \sigma \sqrt{2\log\left(\frac{1}{\delta}\right)+ d\log\left(1+
 \frac{L^2 (1-\gamma^{2(t-1)})}{\lambda d (1-\gamma^2) }\right)}$ \\
\For{$a \in \mathcal{A}_t$}{
Compute $\textnormal{UCB}(a) = a^{\top}\hat{\theta} + \beta_{t-1}  \sqrt{a^{\top} V^{-1} \widetilde{V} V^{-1} a} $
}
$A_t = \argmax_{a}(\textnormal{UCB}(a))$ \\
\textbf{Play action} $A_t$  and \textbf{receive reward} $X_t$ \\
\textbf{Updating phase}:
$V = \gamma V +  A_t A_t^{\top} + (1- \gamma) \lambda I_d$, $\widetilde{V} = \gamma^2 \widetilde{V} + A_t A_t^{\top} + (1-\gamma^2) \lambda I_d$ \\
$ \hspace{2.6cm} b = \gamma b + X_t A_t$, $\hat{\theta} =V^{-1}b$
}
\caption{$\DLinUCB$}
\end{algorithm}


\subsection{Analysis}
\label{sec:analysis}

As discussed previously, we consider weights of the form
$w_t = \gamma^{-t}$ (where $0< \gamma < 1$) in the $\DLinUCB$
algorithm. In accordance with the discussion at the end of
Section~\ref{theorem_deviation_weighted_sequential},
Algorithm~\ref{alg:DLinUCB} uses $\mu_t = \gamma^{-2t}\lambda$ as
the parameter to define the confidence ellipsoid around
$\hat\theta_{t-1}$. The confidence ellipsoid $\cC_t$ is defined as
$\big\{\theta: \lVert \theta - \hat{\theta}_{t-1} \rVert_{V_{t-1} \tV_{t-1}^{-1}
  V_{t-1}} \leq \beta_{t-1} \big\}$ where
\begin{align}
\label{eq:beta_t}
\beta_t = \sqrt{\lambda} S + \sigma \sqrt{2 \log(1/\delta) + d \log\left(1 + \frac{L^2 (1- \gamma^{2t})}{\lambda d (1-\gamma^2)}\right)}.
\end{align}
Using standard algebraic calculations together with the remark above
about scale-invariance it is easily checked that at time $t$
Algorithm~\ref{alg:DLinUCB} selects the action $A_t$ that maximizes
$\langle a, \theta\rangle$ for $a \in \cA_t$ and $\theta\in\cC_t$. The
following theorem bounds the regret resulting from
Algorithm~\ref{alg:DLinUCB}.

\begin{theorem}
\label{th:regret_dLinUCB}
Assuming that $\sum_{s=1}^{T-1} \lVert \theta^{\star}_s - \theta^{\star}_{s+1}
\rVert_2 \leq B_T$, the regret of the $\DLinUCB$ algorithm is bounded for all $\gamma \in (0,1)$
 and integer $D\geq 1$, with probability at least $1-\delta$, by
\begin{equation}
  R_T \leq 2L D B_T  + \frac{4L^3S}{\lambda} \frac{\gamma^{D}}{1-\gamma} T + 2\sqrt{2} \beta_T \sqrt{dT} \sqrt{T \log(1/\gamma) 
  +  \log\left(1+ \frac{L^2}{d\lambda(1-\gamma)}\right)} .
  \label{eq:theorem_regret}
\end{equation}
\end{theorem}

The first two terms of the r.h.s. of~\eqref{eq:theorem_regret}
are the result of the bias due to the non-stationary environment. The
last term is the consequence of the high probability bound established
in the previous section and an adaptation
of the technique used in \cite{abbasi2011improved}.

We give the complete proof of this result in
appendix. The high-level idea of the proof is to isolate bias and variance
terms. However, in contrast with the stationary case, the confidence
ellipsoid $\cC_t$ does not necessarily contain (with high probability)
the actual parameter value $\theta^\star_t$ due to the (unknown) bias
arising from the time variations of the parameter. We thus define
$$\bar\theta_t = V_{t-1}^{-1} \left(\sum_{s=1}^{t-1} \gamma^{-s} A_s A_s^{\top} \theta_s^\star + \lambda \gamma^{-(t-1)} \theta_t^\star \right)$$
which is an action-dependent analogue of the parameter value
$\theta^\star$ in the stationary setting (although this is a random
value). As mentioned in
section~\ref{sec:confidence}, $\bar\theta_t$
does belong to $\cC_t$ with probability at least $1-\delta$ (see Proposition \ref{proposition:bar_theta} in Appendix). The
regret may then be split as
$$
R_T \leq 2 L \sum_{t=1}^T \|\theta_t^\star - \bar\theta_t\|_2 + \sum_{t=1}^T \langle A_t, \theta_t - \bar\theta_t \rangle
 \quad \text{(with probability at least $1-\delta$)},
$$
where
$(A_t, \theta_t) = \arg\max_{(a\in\cA_t,\theta\in\cC_t)} \langle
a,\theta \rangle$. The rightmost term can be handled by proceeding as in
the case of stationary linear bandits, thanks to the deviation inequality
obtained in Section~\ref{sec:confidence}. The first term in the
r.h.s. can be bounded deterministically, from the assumption made on
$\sum_{s=1}^{T-1} \lVert \theta^{\star}_s - \theta^{\star}_{s+1}
\rVert_2$. In doing so, we introduce the analysis parameter $D$ that,
roughly speaking, corresponds to the window length equivalent to a
particular choice of discount factor $\gamma$: the bias resulting from
observations that are less than $D$ time steps apart may be bounded
in term of $D$ while the remaining ones are bounded globally by the
second term of the r.h.s. of~\eqref{eq:theorem_regret}.
This sketch of proof is substantially different from the arguments used by
\cite{cheung2018learning} to analyze their sliding window algorithm
(called $\SWLinUCB$). We refer to the appendix for a more detailed
analysis of these differences. Interestingly, the regret bound of Theorem
\ref{th:regret_dLinUCB} holds despite the fact that the true parameter
$\theta^{\star}_t$ may not be contained in the confidence ellipsoid
$\mathcal{C}_{t-1}$, in contrast to the proof of
 \cite{garivier2011upper}.

 It can be checked that, as $T$ tends to infinity, the optimal choice
 of the analysis parameter $D$ is to take $D = \log(T)/(1-\gamma)$.
 Further assuming that one may tune $\gamma$ as a function of the
 horizon $T$ and the variation upper bound $B_T$ yields the following
 result.

\begin{corollary}
\label{corollary:asymptotic_regret_DLinUCB}
  By choosing $\gamma = 1- (B_T/(dT))^{2/3}$,
 the regret of the $\DLinUCB$ algorithm is asymptotically upper bounded 
 with high probability
 by a term $O(d^{2/3} B_T^{1/3} T^{2/3} )$ when $T \to \infty$.
\end{corollary}

This result is favorable as it corresponds to the same order as the
lower bound established by \cite{besbes2014stochastic}. More
precisely, the case investigated by \cite{besbes2014stochastic}
corresponds to a non-contextual model with a number of changes that
grows with the horizon. On the other hand, the guarantee of
Corollary~\ref{corollary:asymptotic_regret_DLinUCB} requires
horizon-dependent tuning of the discount factor $\gamma$, which opens
interesting research issues (see also \cite{cheung2018learning}).


\section{Experiments}
\label{sec:experiments}

This section is devoted to the evaluation of the empirical performance
of $\DLinUCB$.  We first consider two simulated low-dimensional
environments that illustrate the behavior of the algorithms when
confronted to either abrupt changes or slow variations of the
parameters. The analysis of the previous section, suggests that
$\DLinUCB$ should behave properly in both situations. We then consider
a more realistic scenario in Section~\ref{sec:criteo_data}, where the
contexts are high-dimensional and extracted from a data set of actual
user interactions with a web service.

For benchmarking purposes, we compare $\DLinUCB$ to the Dynamic Linear
Upper Confidence Bound ($\dLinUCB$) algorithm proposed by
\cite{wu2018learning} and with the Sliding Window Linear UCB
($\SWLinUCB$) of \cite{cheung2018learning}. The principle of the
$\dLinUCB$ algorithm is that a master bandit algorithm is in charge of
choosing the best $\LinUCB$ slave bandit for making the
recommendation. Each slave model is built to run in each one of the
different environments.  The choice of the slave model is based on a
lower confidence bound for the so-called \textit{badness} of the
different models. The badness is defined as the number of times the
expected reward was found to be far enough from the actual observed
reward on the last $\tau$ steps, where $\tau$ is a parameter of the
algorithm.  When a slave is chosen, the action proposed to a user is
the result of the $\LinUCB$ algorithm associated with this slave. When
the action is made, all the slave models that were good enough are
updated and the models whose badness were too high are deleted from
the pool of slaves models.  If none of the slaves were found to be sufficiently
good, a new slave is added to the pool.

The other algorithm that we use for comparison is $\SWLinUCB$, as
presented in \cite{cheung2018learning}. Rather than using
exponentially increasing weights, a hard threshold is adopted. Indeed,
the actions and rewards included in the $\length$-length 
sliding window
are used to estimate the linear regression coefficients. We expect
$\DLinUCB$ and $\SWLinUCB$ to behave similarly as they both may be
shown to have the same sort of regret guarantees (see appendix).

In the case of abrupt changes, we also compare these algorithms to the
Oracle Restart LinUCB ($\OmniLinUCB$) strategy that would know the
change-points and simply restart, after each change, a new instance of the
$\LinUCB$ algorithm. The regret of this strategy may be seen as an empirical lower
bound on the optimal behavior of an online learning algorithm in
abruptly changing environments.

In the following figures, the vertical red dashed lines correspond to
the change-points (in abrupt changes scenarios). They are represented
to ease the understanding but except for $\OmniLinUCB$, they are of
course unknown to the learning algorithms.
When applicable, the blue
dashed lines correspond to the average detection time of the breakpoints
with the $\dLinUCB$ algorithm.  For $\DLinUCB$ the discount parameter 
is chosen as
$\gamma = 1 - (\frac{B_T}{dT})^{2/3}$.  For $\SWLinUCB$ the window's
length is set to $\length = (\frac{dT}{B_T})^{2/3}$, where $d=2$ in the experiment.  Those values are
theoretically supposed to minimize the asymptotic regret.  For the
Dynamic Linear UCB algorithm, the badness is estimated from
$\tau = 200$ steps, as in the experimental section of \cite{wu2018learning}.

\subsection{Synthetic data in abruptly-changing or slowly-varying scenarios}

\begin{figure}[!h]
        \centering
        \begin{tabular}{cc}
            \includegraphics[width=0.355\textwidth]{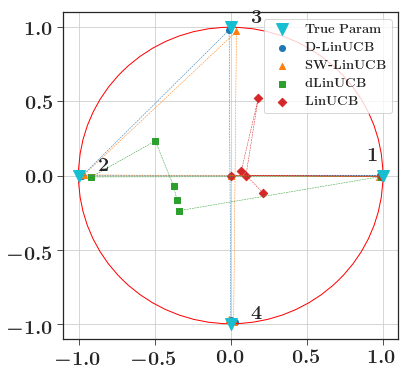} & \hspace*{0.9em} \includegraphics[width=0.355\textwidth]{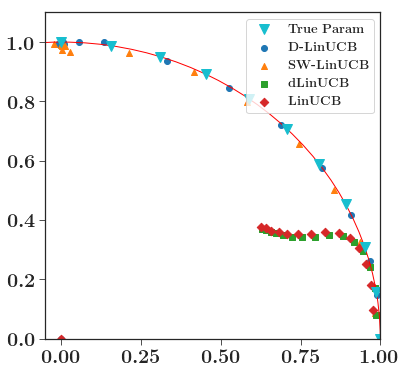}\\
            \includegraphics[width=0.37\textwidth]{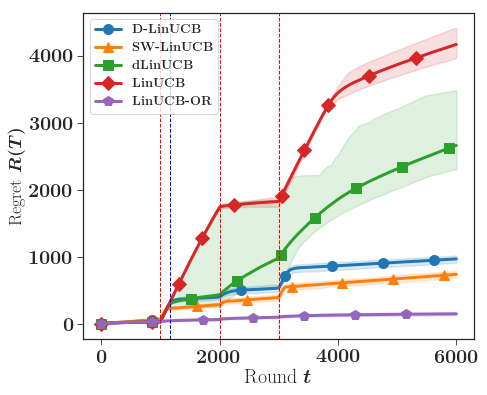} & \includegraphics[width=0.365\textwidth]{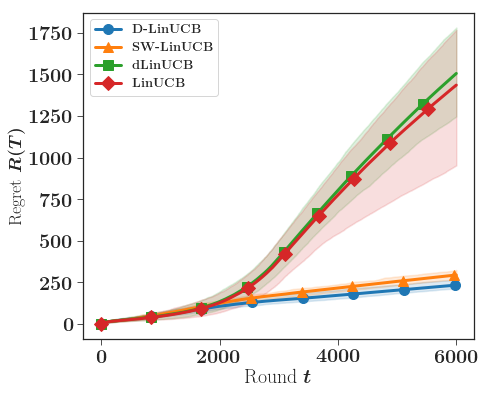}  \\
        \end{tabular}
        \caption{Performances of the algorithms in the abruptly-changing environment (on the left), and, the slowly-varying environment (on the right). 
        The upper plots correspond to the estimated parameter and the lower ones to the accumulated regret, both are averaged on $N=100$ 
        independent experiments}
        \label{tbl:breakpoints_expes}
\end{figure}

In this first experiment, we observe the empirical
performance of all algorithms in an abruptly changing environment of
dimension 2 with 3 breakpoints. The number of rounds is set to
$T=6000$. The light blue triangles correspond to the different
positions of the true unknown parameter $\theta^{\star}_t$: before
$t=1000$, $\theta^{\star}_t = (1,0)$; for
$t \in [\![ 1000,2000]\!], \theta^{\star}_t = (-1,0)$; for
$t \in [\![ 2000,3000]\!], \theta^{\star}_t = (0,1)$; and, finally,
for $t > 3000, \theta^{\star}_t = (0,-1)$.  This corresponds to a hard
problem as the sequence of parameters is widely spread in the unit
ball.  Indeed it forces the algorithm to adapt to big changes, which
typically requires a longer adaptation phase. On the other hand, it
makes the detection of changes easier, which is an advantage for
$\dLinUCB$. In the second half of the experiment (when $t\geq 3000$)
there is no change, $\LinUCB$ struggles to
catch up and suffers linear regret for long periods after the last
change-point.  The results of our simulations are shown in
the left column of Figure~\ref{tbl:breakpoints_expes}.  On the top row
we show a 2-dimensional scatter plot of the estimate of the unknown
parameters $\hat{\theta}_t$ every 1000 steps averaged on
 100 independent experiment. The bottom row
corresponds to the regret averaged over 100 independent
 experiments with the upper and
the lower $5\%$ quantiles.
In this environment, with $1$-subgaussian random noise, $\dLinUCB$
struggles to detect the change-points. Over the 100 experiments, the
first change-point was detected in $95\%$ of the runs, the second
 was never detected and the
third only in $6\%$ of the runs, thus limiting the effectiveness of the
$\dLinUCB$ approach. When decreasing the variance of the noise,
the performance of $\dLinUCB$ improves and gets closer to
 the performance of
the oracle restart strategy $\OmniLinUCB$. It is worth noting that for
both $\SWLinUCB$ and $\DLinUCB$, the estimator $\hat{\theta}_t$ adapts
itself to non-stationarity and is able to follow $\theta^{\star}_t$
(with some delay), as shown on the scatter plot. Predictably,
$\OmniLinUCB$ achieves the best performance by restarting exactly
whenever a change-point happens.
 
The second experiment corresponds to a slowly-changing environment.
It is easier for $\LinUCB$ to keep up with the adaptive policies in
this scenario.  Here, the parameter $\theta^{\star}_t$ starts at
$(1$ and moves continuously counter-clockwise on the unit-circle up to
the position $[0,1]$ in 3000 steps. We then have a steady period of
3000 steps. For
this sequence of parameters,
$B_T = \sum_{t=1}^{T-1} \lVert \theta^{\star}_t -
\theta^{\star}_{t+1}\rVert_2 = 1.57$.  The results are reported in the
right column of Figure~\ref{tbl:breakpoints_expes}.  Unsurprisingly,
$\dLinUCB$ does not detect any change and thus displays the same
performance as $\LinUCB$. $\SWLinUCB$ and $\DLinUCB$ behaves similarly
and are both robust to such an evolution in the regression
parameters. The performance of $\OmniLinUCB$ is not reported here, as
restarting becomes ineffective when the changes are too frequent
(here, during the first 3000 time steps, there is a change at every
single step). The scatter plot
also gives interesting information: $\hat{\theta}_t$ tracks
$\theta^{\star}_t$ quite effectively for both $\SWLinUCB$ and
$\DLinUCB$ but the two others algorithms lag behind. $\LinUCB$ will
eventually catch up if the length of the stationary period becomes
larger.

\subsection{Simulation based on a real dataset}
\label{sec:criteo_data}

\begin{figure}[hbt]
        \centering
        \includegraphics[width=0.41\textwidth]{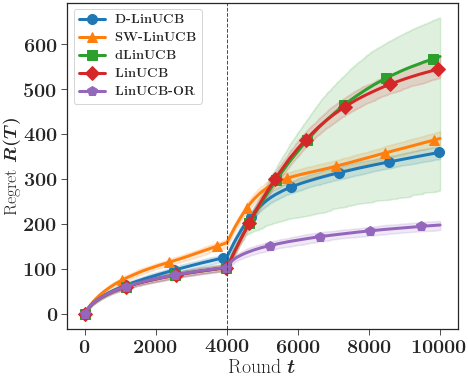}
        \caption{Behavior of the different algorithms on large-dimensional data}
        \label{breakpoints_criteo}
\end{figure}

$\DLinUCB$ also performs well in high-dimensional space ($d=50$). For this
experiment, a dataset
providing a sample of 30 days of Criteo live traffic data \cite{DiemertMeynet2017} was
used. It contains banners that were displayed to different users and
contextual variables, including the information of whether the banner
was clicked or not. We kept the categorical variables $cat1$ to $cat9$
, together  with the variable $campaign$, which is a unique identifier of each campaign. 
Beforehand, these contexts have been  one-hot encoded and $50$ of
the resulting features have been selected using a Singular Value Decomposition. 
$\theta^{\star}$ is obtained by
linear regression. The rewards are then simulated using the regression
model with an additional Gaussian noise of variance $\sigma^2 = 0.15$. At each time step, the different
 algorithms have the choice
between two 50-dimensional contexts drawn at random from two
separate pools of $10000$ contexts corresponding, respectively, to
clicked or not clicked banners. The non-stationarity is created by 
switching $60\%$ of $\theta^{\star}$ coordinates
to $-\theta^{\star}$ at time $4000$, corresponding to
a partial class inversion.  The
cumulative dynamic regret is then averaged over 100 independent
replications. The results are shown on
Figure~\ref{breakpoints_criteo}.  In the first stationary period,
$\LinUCB$ and $\dLinUCB$ perform better than the
adaptive policies by using all available data, whereas the adaptive
 policies only use the most recent events. 
 After the breakpoint, $\LinUCB$ suffers a large
regret, as the algorithm fails to adapt to the new
environment. In this experiment,  $\dLinUCB$ does not detect the 
change-point systematically 
and performs similarly as $\LinUCB$ on average, it can still 
outperform adaptive policies from time to 
time when the breakpoint is detected as can be seen 
with the $5\%$ quantile. $\DLinUCB$ and $\SWLinUCB$ adapt 
more quickly to the change-point 
and perform 
significantly better than the non-adaptive policies after the breakpoint. 
Of course, the oracle policy 
$\OmniLinUCB$ is the best performing policy. The take-away message is 
that there is no free lunch: in a stationary period by using only the 
most recent events
 $\SWLinUCB$ and $\DLinUCB$ do not perform as good as a policy 
 that uses all the available information. 
 Nevertheless, after a breakpoint, the recovery is much faster with 
 the adaptive policies.


\bibliographystyle{abbrvnat}
\bibliography{refs.bib}

\newpage
\appendix

\section*{Appendix}

\section{Confidence Bounds for Weighted Linear Bandits}
\label{Appendix:confidence_bounds}
\subsection{Preliminary results}
In this section we give the main results for obtaining Theorem \ref{theorem_deviation_weighted_sequential}. 
For the sake of conciseness all the results will be stated with $\sigma$-subgaussian noises but the proofs will 
be done with the particular value of $\sigma = 1$. 
The model we consider is the one defined by equation (\ref{eq:reward_generation}), where we recall that $(\eta_s)_s$ is, 
conditionally on the past, a sequence of $\sigma$-subgaussian random noises. The results of this section are close
to the one proposed in \citep{abbasi2011improved} but our results are valid with a sequence of predictable weights.

We introduce the quantity $S_t = \sum_{s=1}^t w_s A_s \eta_s$ and $\tV_t = \sum_{s=1}^t w_s^2 A_s A_s^{\top} + \mu_t I_d$. 
When the regularization term is omitted, let $\tV_t(0) = \sum_{s=1}^t w_s^2 A_s A_s^{\top}$. 
The filtration associated with the random observations is denoted $\mathcal{F}_t =\sigma(X_1,\dots,X_t)$ 
such that $A_t$ is $\mathcal{F}_{t-1}$-measurable and $\eta_t$ is $\mathcal{F}_t$-measurable. 
The weights are also assumed to be predictable. The following lemma is an extension 
to the weighted case of Lemma 8 of \citep{abbasi2011improved}.
\begin{lemma}
\label{lemma:exp_upper}
Let $(w_t)_{t \geq 1}$ be a sequence of predictable and positive weights. Let $x \in \mathbb{R}^d$ be arbitrary and consider for any $t \geq 1$
$$
M_t(x)  = \exp\left(\frac{1}{\sigma}x^{\top} S_t - \frac{1}{2} x^{\top} \tV_t(0) x \right).
$$ 

Let $\tau$ be a stopping time with respect to the filtration $\{\mathcal{F}_t\}_{t=0}^{\infty}$. Then $M_{\tau}(x)$ is almost surely well-defined and 
$$
\forall x \in \mathbb{R}^d, \mathbb{E}\lbrack M_{\tau}(x)\rbrack \leq 1.
$$
\end{lemma}
\begin{proof}
First, we prove that $\forall x \in \mathbb{R}^d, (M_t(x))_{t=0}^{\infty}$ is a super-martingale. 

Let $x \in \mathbb{R}^d$,
\begin{align*}
\mathbb{E}\lbrack M_t(x)| \mathcal{F}_{t-1}\rbrack &= \mathbb{E}\left\lbrack \exp\left( x^{\top} S_{t-1}+ x^{\top} w_t A_t \eta_t -   
1/2 x^{\top} ( \tV_{t-1}(0) + w_t^2 A_t A_t^{\top}) x
  \right)| \mathcal{F}_{t-1}\right\rbrack \\
  &= M_{t-1}(x) \mathbb{E}\left \lbrack \exp( x^{\top}  w_t A_t \eta_t -\frac{1}{2} w_t^2 x^{\top} A_t A_t^{\top} x)| \mathcal{F}_{t-1} \right\rbrack  \\
  &= M_{t-1}(x) \exp(-\frac{1}{2} w_t^2 x^{\top} A_t A_t^{\top} x) \mathbb{E}\left \lbrack \exp( x^{\top}  w_t A_t \eta_t)| \mathcal{F}_{t-1} \right\rbrack \\
  & \leq M_{t-1}(x) \exp(-\frac{1}{2} w_t^2 x^{\top} A_t A_t^{\top} x) \exp(1/2 w_t^2 (x^{\top} A_t )^2) \\
  & = M_{t-1}(x).
\end{align*}
The second equality comes from the fact that $S_{t-1}$ and $\tV_{t-1}$ are $\mathcal{F}_{t-1}$-measurable. 
The inequality is the definition of the conditional $1$-subgaussianity where we also use the $\mathcal{F}_{t-1}$-measurability of $w_t$.

Using this supermartingale property, we have $\mathbb{E}\lbrack M_t(x) \rbrack \leq 1$. 
The convergence theorem for non-negative supermartingales ensures that  $M_{\infty}(x) = \lim_{t \to \infty} M_t(x)$ is almost surely well defined. 
By introducing the stopped supermartingale $\mathcal{M}_t(x) = M_{\min(t,\tau)}(x)$, we have $M_\tau(x) = \lim_{t \to \infty} \mathcal{M}_t(x)$.
Knowing that $\mathcal{M}_t(x)$ is also a supermartingale, we have
$$
\mathbb{E} \lbrack \mathcal{M}_t(x) \rbrack = \mathbb{E} \lbrack M_{\min(t,\tau)}(x) \rbrack \leq \mathbb{E} \lbrack M_{\min(0,\tau)}(x) \rbrack = \mathbb{E} \lbrack M_{0}(x) \rbrack  = 1.
$$
By using Fatou's lemma:
$$
\mathbb{E} \lbrack M_{\tau}(x) \rbrack = \mathbb{E} \lbrack \liminf_{t\to \infty} \mathcal{M}_{t}(x) \rbrack \leq \liminf_{t \to \infty}{\mathbb{E}\lbrack \mathcal{M}_{t}(x) \rbrack} \leq 1.
$$
\end{proof}
In the next lemma, we will integrate $M_t(x)$ with respect to a time-dependent probability measure. 
This is the key for allowing sequential regularizations in the concentration inequality stated in Theorem \ref{theorem_deviation_weighted_sequential}.
 This lemma is inspired by the method of mixtures first presented in \cite{pena2008self}.
The idea of using time-varying probability measures is inspired from the proof of Theorem 11 in \cite{kirschner2018information}. 
The two following lemmas are included in the appendix so that the article is self-contained. There are not a mere consequence of the results in
 \cite{abbasi2011improved} because of the time-dependent regularization parameters. As explained in Section \ref{sec:non-stationary}, this is unavoidable when using exponential weights to avoid the
vanishing effect of the regularization. 
\begin{lemma}
\label{lemma:tilde}
Let $(h_t)_t$  be a sequence of probability measures on $\mathbb{R}^d$. We define $\widetilde{M}_t = \int_{\mathbb{R}^d} M_t(x) d h_t(x)$. Then,
$$
\forall t, \mathbb{E} \lbrack \widetilde{M}_t \rbrack \leq 1
$$
\end{lemma}

\begin{proof}
\begin{align*}
    \mathbb{E}\lbrack \widetilde{M}_t \rbrack = \int \widetilde{M}_t \, d \mathbb{P}    & = \int \left(\int_{\mathbb{R}^d} M_t(x) dh_t(x) \right) d\mathbb{P} \\
    &= \int_{\mathbb{R}^d} \left(\int M_t(x) d\mathbb{P} \right) dh_t(x) \quad \textnormal{(Fubini's theorem)} \\
    &= \int_{\mathbb{R}^d} \mathbb{E} \lbrack M_t(x)\rbrack dh_t(x) \\ 
    &\leq \int_{\mathbb{R}^d} dh_t(x) \quad \textnormal{(Lemma \ref{lemma:exp_upper})} \\
    &\leq 1.   \quad \textnormal{($h_t$ probability measure.)}
 \end{align*}
\end{proof}
Lemma \ref{lemma:tilde} is a warm-up for the next lemma and is helpful for understanding why Lemma \ref{lemma_stopped_evolving} holds. It is valid for any fixed time $t$.
 The next step is to give its equivalent in a stopped version in the specific case of gaussian random vectors.
\begin{lemma}
\label{lemma_stopped_evolving}
Let $(\mu_t)_t$ be a deterministic sequence of regularization parameters. Let $\mathcal{F}_{\infty}= \sigma\left( \cup_{t=1}^{\infty} \mathcal{F}_t \right)$ be the tail $\sigma$-algebra of the filtration
 $(\mathcal{F}_t)_t$. Let $X = (X_t)_{t \geq 1}$ be an independent sequence of gaussian 
 random vectors such that $X_t \sim \mathcal{N}(0, \frac{1}{\mu_t}I_d) = h_t$ with $X$ 
 independent of $\mathcal{F}_{\infty}$.
We define
$$
\bar{M}_t(\mu_t) = \mathbb{E} \lbrack M_t(X_t) | \mathcal{F}_{\infty} \rbrack  = \int_{\mathbb{R}^d} M_t(x) f_{\mu_t}(x) dx,
$$
where  $f_{\mu_t}$ is the probability density function associated with $h_t$ defined as,
$$ f_{\mu_t}(x) = \frac{1}{\sqrt{(2\pi)^d \det(1/\mu_t I_d)}}\exp(-\frac{\mu_t x^{\top}x}{2}).$$ 
Let  $\tau$ be a stopping time with respect to the filtration $(\mathcal{F}_t)_t$ then,
$$
\mathbb{E} \lbrack \bar{M}_{\tau}(\mu_{\tau}) \rbrack \leq 1.
$$
\end{lemma}
\begin{proof}
We can use the result of Lemma \ref{lemma:exp_upper} which gives $\forall x \in \mathbb{R}^d , \, \mathbb{E} \lbrack M_{\tau}(x)\rbrack \leq 1$.

We have,
\begin{align*}
    \mathbb{E} \lbrack \bar{M}_{\tau}(\mu_{\tau}) \rbrack &= \mathbb{E} \lbrack \mathbb{E} \lbrack M_{\tau}(X_{\tau})| \mathcal{F}_{\infty}\rbrack \rbrack 
    =  \mathbb{E} \lbrack \mathbb{E} \lbrack \mathbb{E} \lbrack M_{\tau}(X_{\tau})| \mathcal{F}_{\infty}\rbrack | (X_t)_{t\geq 1} \rbrack \rbrack \\
    & = \mathbb{E} \lbrack \mathbb{E} \lbrack \mathbb{E} \lbrack M_{\tau}(X_{\tau})| (X_t)_{t\geq 1} \rbrack | \mathcal{F}_{\infty} \rbrack \rbrack  
    \leq 1.
\end{align*}

The inequality is a consequence of Lemma \ref{lemma:exp_upper} as, 
conditionally to the sequence  $(X_t)_t$, $M_{\tau}(X_{\tau})$ is of the 
form $M_{\tau}(x)$ with a fixed $x$.
\end{proof}

We finally state the main result needed to obtain Theorem \ref{theorem_deviation_weighted_sequential}.
\begin{proposition} For $(w_s)_{s\geq 1}$ a sequence of predictable and positive weights,  $\forall  \delta > 0 $, the following deviation inequality holds
\label{prop:deviation}
\begin{align*}
\mathbb{P}\left(\exists t \geq 0,\lVert S_t \rVert_{\widetilde{V}_t^{-1}}\geq  \sigma\sqrt{ 2\log\left(\frac{1}{\delta}\right) +
 \log\left(\frac{\det(\widetilde{V}_t)}{\mu_t^d}\right)}   \right) \leq \delta.
\end{align*}
\end{proposition}
\begin{proof}
For a fixed $t$,
\begin{align*}
    \Bar{M}_{t}(\mu_t) &= \int_{\mathbb{R}^d} M_{t}(x) f_{\mu_t}(x) dx \\
    &= \frac{1}{\sqrt{(2 \pi)^d \det(1/\mu_t I_d)}}\int_{\mathbb{R}^d} \exp\left( x^{\top}S_{t} -
    \frac{1}{2}\Vert x \Vert_{\mu_t I_d}^2-\frac{1}{2}\Vert x \Vert_{\widetilde{V}_{t}(0)}^2 \right) dx \\
     &= \frac{1}{\sqrt{(2 \pi)^d \det(1/\mu_t I_d)}}\int_{\mathbb{R}^d} \exp\left( x^{\top}S_{t} 
   -\frac{1}{2}\Vert x \Vert_{\widetilde{V}_{t}}^2 \right) dx \\
    &= \frac{1}{\sqrt{(2 \pi)^d \det(1/\mu_t I_d)}}\int_{\mathbb{R}^d}\exp\left(\frac{1}{2}\Vert S_{t}\Vert_{\widetilde{V}_t^{-1}}^2 - 
    \frac{1}{2}\Vert x - \widetilde{V}_t^{-1}S_{t}\Vert_{\widetilde{V}_t}^2\right) dx \\
    &= \frac{\exp\left(\frac{1}{2}\Vert S_{t}\Vert_{\widetilde{V}_t^{-1}}^2\right)}{\sqrt{(2 \pi)^d \det(1/\mu_t I_d)}}\int_{\mathbb{R}^d}\exp\left( - 
    \frac{1}{2}\Vert x - \widetilde{V}_t^{-1}S_{t}
    \Vert_{\widetilde{V}_t}^2\right) dx \\
    &= \frac{\exp\left(\frac{1}{2}\Vert S_{t}\Vert_{\widetilde{V}_t^{-1}}^2\right)}{\sqrt{(2 \pi)^d \det(1/\mu_t I_d)}} \sqrt{(2\pi)^d
     \det\left(\widetilde{V}_t^{-1}\right)} \\
    &= \exp\left(\frac{1}{2}\Vert S_{t}\Vert_{\widetilde{V}_t^{-1}}^2\right) \sqrt{\frac{\det(\mu_t I_d)}{\det(\widetilde{V}_t)}}.
\end{align*}
We introduce the particular stopping time,
$$
\tau = \min \bigg\{t\geq 0,  \lVert S_{t} \rVert_{\widetilde{V}_{t}^{-1}} \geq \sqrt{2\log\left(\frac{1}{\delta}\right) + 
\log\left(\frac{\det(\tV_{t})}{\det(\mu_{t}I_d)}\right)} \bigg\}.
$$
Thus,
\begin{align*}
    &
    \mathbb{P}\left(\exists t \geq 0, \lVert S_{t} \rVert_{\widetilde{V}_{t}^{-1}} \geq \sqrt{2\log\left(\frac{1}{\delta}\right) + 
    \log\left(\frac{\det(\tV_{t})}{\det(\mu_{t}I_d)}\right)} \right) 
    =  \mathbb{P}(\tau < \infty) \\
    & \,\, 
    = \mathbb{P}\left( \tau < \infty, \lVert S_{\tau} \rVert_{\widetilde{V}_{\tau}^{-1}} \geq \sqrt{2\log\left(\frac{1}{\delta}\right) +
     \log\left(\frac{\det(\tV_{\tau})}{\det(\mu_{\tau}I_d)}\right)} \right) \\
    & \,\, 
     \leq \mathbb{P}\left(  \lVert S_{\tau} \rVert_{\widetilde{V}_{\tau}^{-1}} \geq \sqrt{2\log\left(\frac{1}{\delta}\right) + 
     \log\left(\frac{\det(\tV_{\tau})}{\det(\mu_{\tau}I_d)}\right)} \right) \\
    & \,\, 
    =  \mathbb{P}\left( \exp\left(\frac{1}{2}\Vert S_{\tau}\Vert_{\widetilde{V}_{\tau}^{-1}
    }^2\right) \sqrt{\frac{\det(\mu_{\tau} I_d)}{\det(\widetilde{V}_{\tau})}} \geq \frac{1}{\delta}\right) \\
     & \,\,
      \leq  \delta \mathbb{E}\lbrack \bar{M}_{\tau}(\mu_{\tau})\rbrack \,\, \textnormal{(Markov's inequality)}
     \leq \delta \, \, \textnormal{(Lemma \ref{lemma_stopped_evolving})}.
\end{align*}
\end{proof}
\subsection{Proof of Theorem \ref{theorem_deviation_weighted_sequential}}
We recall that Theorem \ref{theorem_deviation_weighted_sequential} is established in a stationary environment where 
$\forall t \geq 1, \, \theta^{\star}_t = \theta^{\star}$.
\begin{proof}
First note that, 
\begin{align*}
\hat{\theta}_t  &= V_t^{-1} \sum_{s=1}^t w_s A_s X_s\\
& = V_t^{-1} \sum_{s=1}^t w_s A_s (A_s^{\top} \theta^{\star}+ \eta_s) \quad (\textnormal{Equation}  \, \ref{eq:reward_generation}) \\
& =  V_t^{-1} \left(\sum_{s=1}^t w_s A_s A_s^{\top} \theta^{\star} + \lambda_t \theta^{\star} - \lambda_t \theta^{\star}\right) + V_t^{-1} S_t 
= \theta^{\star}   - \lambda_t  V_t^{-1} \theta^{\star} + V_t^{-1} S_t.
\end{align*}
Thus,
\begin{align}
\label{link_thetas}
\hat{\theta}_t - \theta^{\star} = V_t^{-1} S_t - \lambda_t  V_t^{-1} \theta^{\star}.
\end{align}
$\forall x \in \mathbb{R}^d, \forall t >0$, we have
\begin{align*}
\lvert x^{\top}(\hat{\theta}_t - \theta^{\star}) \rvert &\leq \lVert x \rVert_{V_t^{-1}\tV_tV_t^{-1}} \left( \lVert V_t^{-1} S_t \rVert_{V_t \tV_t^{-1} V_t}  +
\lVert \lambda_t V_t^{-1}\theta^{\star} \rVert_{V_t \tV_t^{-1} V_t}\right) \\ 
& \leq \lVert x \rVert_{V_t^{-1}\tV_tV_t^{-1}} \left( \lVert S_t \rVert_{ \tV_t^{-1}}  + \lambda_t \lVert  \theta^{\star} \rVert_{ \tV_t^{-1}}\right) .
\end{align*}
By applying the previous inequality with $ x = V_t \tV_t^{-1} V_t
 (\hat{\theta}_t- \theta^{\star})$, we have
\begin{align*}
\forall t, \lVert \hat{\theta}_t - \theta^{\star} \rVert_{V_t \widetilde{V}_t^{-1} V_t} \leq \lVert S_t \rVert_{\widetilde{V}_t^{-1}} +
 \lambda_t  \lVert \theta^{\star} \rVert_{\widetilde{V}_t^{-1}} .
\end{align*}
Knowing that $\widetilde{V}_t \geq \mu_t I_d$ and that $\tV_t $ is positive definite, 
we have $ \lVert \theta^{\star} \rVert_{\widetilde{V}_t^{-1}} \leq \frac{1}{\sqrt{\mu_t}}\lVert \theta^{\star}\rVert_2$.

Finally,
\begin{align}
\label{ineq_concentration}
    \forall t, \lVert \hat{\theta}_t - \theta^{\star} \rVert_{V_t \widetilde{V}_t^{-1} V_t} \leq \lVert S_t \rVert_{\widetilde{V}_t^{-1}} + 
    \frac{\lambda_t}{\sqrt{\mu_t}}  \lVert \theta^{\star} \rVert_{2}.
\end{align}
From Proposition \ref{prop:deviation}, we obtain the following any time high probability upper
 bound for $\lVert S_t\rVert_{\widetilde{V}_t^{-1}}$, 
\begin{align*}
\mathbb{P}\left(\forall t \geq 0,\lVert S_t \rVert_{\widetilde{V}_t^{-1}}\leq  
\sigma\sqrt{ 2\log\left(\frac{1}{\delta}\right) + \log\left(\frac{\det(\widetilde{V}_t)}{\mu_t^d}\right)}   \right) 
\geq 1- \delta.
\end{align*}
Therefore by using inequality \ref{ineq_concentration},
\begin{align*}
\mathbb{P}\left(\forall t \geq 0,\lVert \hat{\theta}_t - \theta^{\star} \rVert_{\widetilde{V}_t^{-1}}\leq \frac{\lambda_t}{\sqrt{\mu_t}}S + 
\sigma\sqrt{ 2\log\left(\frac{1}{\delta}\right) + \log\left(\frac{\det(\widetilde{V}_t)}{\mu_t^d}\right)}   \right) \geq 1- \delta.
\end{align*}
We obtain the exact formula of Theorem \ref{theorem_deviation_weighted_sequential} by upper bounding $\det(\tV_t)$
 as proposed in Proposition \ref{ineq_det_DLINUCB}
\end{proof}
\section{$\DLinUCB$ Analysis}
In this section, the environment is non-stationary, which means 
that the unknown parameter $\theta^{\star}$ may evolve 
over time and is denoted $\theta^{\star}_t$. The reward generation process in 
the one presented in Equation~(\ref{eq:reward_generation}).
\subsection{Preliminary results}

In this section, $V_t$ and $\tV_t$ are defined by
$$
V_t = \sum_{s=1}^t \gamma^{-s} A_s A_s^{\top} + 
\lambda \gamma^{-t} I_d, \quad
\tV_t = \sum_{s=1}^t \gamma^{-2s}
 A_s A_s^{\top} + \lambda \gamma^{-2t} I_d.
$$
We recall the definition of $\beta_t$:
$$
\beta_t = \sqrt{\lambda} S + \sigma \sqrt{2 \log(1/\delta) + d \log\left(1 + \frac{L^2 (1- \gamma^{2t})}{\lambda d (1-\gamma^2)}\right)}.
$$
With $\hat{\theta}_t$ defined in equation (\ref{eq:thetahatw}), the confidence
 ellipsoid we consider is defined by
\begin{equation}
    \label{confidence_ellipsoid}
    \mathcal{C}_t = \bigg\{\theta \in \mathbb{R}^d: \lVert \theta - \hat{\theta}_{t-1} \rVert_{V_{t-1}\widetilde{V}_{t-1}^{-1} V_{t-1}} \leq \beta_{t-1} \bigg\}.
\end{equation}
Theorem \ref{theorem_deviation_weighted_sequential} can be applied with this choice of 
weights and regularization. We combine it with an upper bound for $\det(\tV_t)$ given below.

\begin{proposition}[Determinant inequality for the weighted design matrix]
\label{ineq_det_DLINUCB}
Let $(\lambda_t)_t$ be a deterministic sequence of regularization parameters. Let $V_t = \sum_{s=1}^t w_s A_s A_s^{\top} + \lambda_t I_d$ be the weighted design matrix. 
Under the assumption $\forall t, \lVert A_t \rVert_2 \leq L$, the following holds
$$
\det(V_t) \leq  \left(\lambda_t + \frac{ L^2 \sum_{s=1}^t w_s}{d}\right)^d.
$$
\end{proposition}
\begin{proof}
\begin{align*}
\det(V_t) &= \prod_{i=1}^d l_i \quad (l_i \textnormal{ are the eigenvalues}) \\
&\leq \left(\frac{1}{d}\sum_{i=1}^d l_i \right)^d  \quad (\textnormal{AM-GM inequality})\\
&\leq \left(\frac{1}{d} \textnormal{trace}(V_t) \right)^d 
\leq \left(\frac{1}{d} \sum_{s=1}^t w_s \textnormal{trace}(A_s A_s^{\top}) + \lambda_t\right)^d \\
&\leq \left(\frac{1}{d} \sum_{s=1}^t w_s\Vert A_s \Vert_2^2 + \lambda_t\right)^d 
\leq \left(\lambda_t + \frac{L^2}{d} \sum_{s=1}^t w_s \right)^d. 
\end{align*}
\end{proof}
\begin{corollary}
\label{corollary:inequality_determinant}
In the specific case where the weights are given by $w_t = \gamma^{-t}$ with $0< \gamma <1$. Proposition \ref{ineq_det_DLINUCB} can be rewritten
\end{corollary}
$$
\det(V_t) \leq \left(\lambda_t + \frac{L^2 (\gamma^{-t}-1)}{d(1-\gamma)} \right)^d = \left( \lambda \gamma^{-t} +
 \frac{L^2 (\gamma^{-t}-1)}{d(1-\gamma)} \right)^d.
$$
We also have,
$$
\det(\tV_t) \leq \left(\mu_t + \frac{L^2 (\gamma^{-2t}-1)}{d(1-\gamma^2)} \right)^d = \left( \lambda \gamma^{-2t} +
 \frac{L^2 (\gamma^{-2t}-1)}{d(1-\gamma^2)} \right)^d.
$$
\begin{proof}
Apply Proposition \ref{ineq_det_DLINUCB} and use $\sum_{s=1}^{t} \gamma^{-s} = \frac{\gamma^{-t}-1}{1-\gamma}$ and $\sum_{s=1}^{t} \gamma^{-2s} = \frac{\gamma^{-2t}-1}{1-\gamma^2}$ .
\end{proof}

Corollary \ref{corollary:inequality_determinant} and Proposition \ref{prop:deviation} yield the
following result.

\begin{corollary}
$\forall \delta > 0$, with the weights $w_t = \gamma^{-t}$ and $0 < \gamma < 1$,  we have
\label{corollary:S_t}
\begin{align*}
\mathbb{P}\left(\exists t \geq 0,\lVert S_t \rVert_{\widetilde{V}_t^{-1}}\geq \sigma\sqrt{ 2\log\left(\frac{1}{\delta}\right) +d \log\left(1 +
\frac{L^2 (1-\gamma^{2t})}{\lambda d (1-\gamma^2)}\right)}   \right) \leq \delta.
\end{align*}
\end{corollary}
Thanks to this corollary we are now ready to show that $\bar{\theta}_t$ belongs to $\mathcal{C}_{t-1}$ with high probability.
\begin{proposition}
\label{proposition:bar_theta}
Let $\mathcal{C}_t =\bigg\{\theta \in \mathbb{R}^d: \lVert \theta - \hat{\theta}_{t-1} \rVert_{V_{t-1}\widetilde{V}_{t-1}^{-1} 
V_{t-1}} \leq \beta_{t-1} \bigg\} $ denote the confidence ellipsoid.
Let $\bar\theta_t = V_{t-1}^{-1} \left(\sum_{s=1}^{t-1} \gamma^{-s} 
 A_s A_s^{\top} \theta_s^\star + \lambda \gamma^{-(t-1)} \theta_t^\star \right)$.
 Then, $\forall \delta >0$, 
 $$
\mathbb{P}\left( \forall t \geq 1, \bar\theta_t \in \mathcal{C}_t\right)\geq 1- \delta.
 $$
\end{proposition}
\begin{proof}
\begin{align*}
 \bar \theta_{t} - \hat{\theta}_{t-1}  &= V_{t-1}^{-1}\left( \sum_{s=1}^{t-1} \gamma^{-s}A_s A_s^{\top}\theta_{s}^{\star} + \lambda 
 \gamma^{-(t-1)} \theta^{\star}_t -  \sum_{s=1}^{t-1} \gamma^{-s}A_s X_s\right)  \\
 &=V_{t-1}^{-1}\left( \sum_{s=1}^{t-1} \gamma^{-s}A_s A_s^{\top}\theta_{s}^{\star} + \lambda \gamma^{-(t-1)} \theta^{\star}_t - 
 \sum_{s=1}^{t-1} \gamma^{-s}A_s A_s^{\top} \theta_{s}^{\star} - \sum_{s=1}^{t-1} \gamma^{-s}A_s \eta_s\right) \\
 &= - V_{t-1}^{-1} S_{t-1} + \lambda \gamma^{-(t-1)} V_{t-1}^{-1} \theta_t^{\star}.
\end{align*}
Therefore,
\begin{align*}
\lVert \bar \theta_{t} - \hat{\theta}_{t-1}\rVert_{V_{t-1} \tV_{t-1}^{-1} V_{t-1}} &\leq \lVert S_{t-1} \rVert_{\tV_{t-1}^{-1}} + \lambda \gamma^{-(t-1)}  \lVert\theta_t^{\star}\rVert_{\tV_{t-1}^{-1}} \\
&\leq \lVert S_{t-1} \rVert_{\tV_{t-1}^{-1}} + \sqrt{\lambda} S \quad \textnormal{($\tV_{t-1}^{-1} \leq 1/(\gamma^{-2(t-1)}\lambda) I_d$  and $\lVert \theta^{\star}_t \rVert_2 \leq S$)} \\
&\leq \beta_{t-1} \quad \textnormal{(Corollary \ref{corollary:S_t})}.
\end{align*}
\end{proof}
\subsection{Control of the norm of actions}
\label{sub-sec:elliptical-lemma}

\begin{lemma}
\label{lemma_ineq_mat}
Let $V_t = \sum_{s=1}^t  \gamma^{-s} A_s A_s^{\top} + \lambda \gamma^{-t} I_d$ and $\widetilde{V}_t= \sum_{s=1}^t 
\gamma^{-2s} A_s A_s^{\top} + \lambda\gamma^{-2t} I_d$ and $0<\gamma<1$. We have
$$
\forall t, \, V_t^{-1}\widetilde{V}_t \,V_t^{-1} \leq \gamma^{-t} \, V_t^{-1}.
$$
\end{lemma}
\begin{proof}
\begin{align*}
    \widetilde{V}_t &= \sum_{s=1}^t \gamma^{-2s} A_s A_s^{\top} + \lambda \gamma^{-2t} I_d  
    \leq \gamma^{-t} \sum_{s=1}^t \gamma^{-s} A_s A_s^{\top} + \lambda \gamma^{-2t} I_d = \gamma^{-t} V_t.
\end{align*}
Consequently,
$$
V_t^{-1}\tV_t V_t^{-1} \leq \gamma^{-t} V_t^{-1} V_t V_t ^{-1} \leq \gamma^{-t}  V_t^{-1}.
$$
\end{proof}
Thanks to Lemma \ref{lemma_ineq_mat} we establish the following proposition,
\begin{proposition}
$$
\sum_{t=1}^T  \min\left(1,\lVert A_t\rVert_{V_{t-1}^{-1} \tV_{t-1} V_{t-1}^{-1}}^2\right) \leq 2 \sum_{t=1}^T  \log\left( 1 + 
\gamma^{-t} \lVert A_t\rVert_{V_{t-1}^{-1}}^2\right) \leq 2\log\left( \frac{\det(V_T)}{\lambda^d}\right).
$$
\end{proposition}
\begin{proof}
We first use the fact that: $\forall x \geq 0, \min(1,x) \leq 2 \log(1+x)$.
\begin{align*}
\min\left(1,\lVert A_t\rVert_{V_{t-1}^{-1} \tV_{t-1} V_{t-1}^{-1}}^2\right) &\leq 2 \log\left(1 + \lVert A_t\rVert_{V_{t-1}^{-1} \tV_{t-1} V_{t-1}^{-1}}^2\right) \\
&\leq 2 \log\left(1 + \gamma^{-(t-1)} \lVert A_t\rVert_{ V_{t-1}^{-1}}^2\right) \quad \textnormal{(Lemma \ref{lemma_ineq_mat})}\\
&\leq 2 \log\left(1 + \gamma^{-t} \lVert A_t\rVert_{ V_{t-1}^{-1}}^2\right) \quad \textnormal{($\gamma \leq 1$)}.\\
\end{align*}
Furthermore, 
\begin{align*}
V_t \geq \gamma^{-t}A_t A_t^{\top} + V_{t-1} \geq  V_{t-1}^{1/2} (I_d + \gamma^{-t} V_{t-1}^{-1/2} A_t A_t^{\top} V_{t-1}^{-1/2})V_{t-1}^{1/2}.
\end{align*}
Given that all those matrices are symmetric positive definite, the previous inequality 
implies that
\begin{align*}
\det(V_t) &\geq \det(V_{t-1}) \det(1 + (\gamma^{-t/2} V_{t-1}^{-1/2} A_t) (\gamma^{-t/2} V_{t-1}^{-1/2} A_t)^{\top}) \\
&\geq \det(V_{t-1}) \left( 1 + \gamma^{-t} \lVert A_t\rVert_{ V_{t-1}^{-1}}^2 \right) \quad \left(\textnormal{Using $\det(I_d + x x^{\top}) = 1 + \lVert x \rVert_2^2  $}\right).
\end{align*}
Therefore, 
\begin{align*}
\frac{\det(V_T)}{\det(V_0)} = \prod_{t=1}^T \frac{\det(V_t)}{\det(V_{t-1})} \geq \prod_{t=1}^T (1 + \gamma^{-t} \lVert A_t\rVert_{ V_{t-1}^{-1}}^2).
\end{align*}
Finally by applying the log function to the previous inequality, 
$$
\sum_{t=1}^T  \min\left(1,\lVert A_t\rVert_{V_{t-1}^{-1} \tV_{t-1} V_{t-1}^{-1}}^2\right) \leq 2 \sum_{t=1}^T \log\left(1 + 
\gamma^{-t} \lVert A_t\rVert_{ V_{t-1}^{-1}}^2\right) \leq 2 \log\left(\frac{\det(V_T)}{\det(V_0)}\right).
$$
\end{proof}
\begin{corollary}
\label{corollary:lemma-elliptic}
$$
\sqrt{\sum_{t=1}^T  \min\left(1,\lVert A_t\rVert_{V_{t-1}^{-1} \tV_{t-1} V_{t-1}^{-1}}^2\right)} \leq \sqrt{2d} \sqrt{T\log\left(\frac{1}{\gamma}\right) +
 \log\left(1 + \frac{L^2}{d \lambda (1-\gamma)}\right)}.
$$
\end{corollary}
\begin{proof}
The proof of this corollary is based on the previous lemma and on Corollary \ref{corollary:inequality_determinant}.
We have
\begin{align*}
\log\left(\frac{\det(V_T)}{\det(V_0)}\right) &\leq \log\left(\frac{1}{\lambda^d}\left(\lambda \gamma^{-T} + 
\frac{L^2(\gamma^{-T} -1)}{d(1-\gamma)}\right)^d  \right) \quad \textnormal{(Corollary \ref{corollary:inequality_determinant})} \\
& \leq dT\log\left(\frac{1}{\gamma}\right) + d \log\left( 1+ \frac{L^2}{d \lambda(1-\gamma)} \right).
\end{align*}
\end{proof}
\subsection{Proof of Theorem \ref{th:regret_dLinUCB}}
In this subsection we give the proof of Theorem \ref{th:regret_dLinUCB} for the high probability upper-bound of the regret for $\DLinUCB$. 
\begin{proof} 
\hspace{0.2cm}

\underline{First step:} Upper bound for the instantaneous regret. 

Let $ A_t^{\star} = \argmax_{a \in \mathcal{A}_t} \langle a, \theta^{\star}_t\rangle$ and
 $\theta_t = \argmax_{\theta \in \mathcal{C}_t} \langle A_t, \theta \rangle$.
We have,
\begin{align*}
r_t &=  \max_{a \in \mathcal{A}_t} \langle a, \theta^{\star}_t \rangle - \langle A_t, \theta^{\star}_t \rangle  = \langle A_t^{\star} - A_t, \theta_t^{\star} \rangle \\
&=  \langle A_t^{\star} - A_t, \bar\theta_t \rangle + \langle A_t^{\star} - A_t, \theta_t^{\star} - \bar\theta_t \rangle .
\end{align*}
Under the event $\{ \forall t >0, \, \bar\theta_t \in \mathcal{C}_t\}$, 
that occurs with probability at least $1-\delta$ thanks to Proposition \ref{proposition:bar_theta}, we have,
\begin{align}
\label{ineq:UCBs}
\langle A_t^{\star}, \bar \theta_t \rangle \leq \argmax_{\theta \in \mathcal{C}_t} \langle A_t^{\star}, \theta \rangle = 
\text{UCB}_t(A_t^{\star}) \leq \text{UCB}_t(A_t) =  \argmax_{\theta \in \mathcal{C}_t} \langle A_t, \theta \rangle  = \langle A_t, \theta_t \rangle.
\end{align}
Then, with probability at least $1-\delta$, $\forall t>0$,
\begin{align*}
r_t &\leq \langle A_t, \theta_t - \bar\theta_t \rangle +  \langle A_t^{\star} - A_t, \theta_t^{\star} - \bar\theta_t \rangle \\
&\leq \lVert A_t\rVert_{V_{t-1}^{-1} \tV_{t-1} V_{t-1}^{-1}} \lVert \theta_t - \bar\theta_t \rVert_{V_{t-1} \tV_{t-1}^{-1} V_{t-1}} + 
\lVert A_t^{\star} - A_t \rVert_2 \lVert \theta_t^{\star} - \bar\theta_t  \rVert_2 \quad \textnormal{(Cauchy-Schwarz)} \\
&\leq \lVert A_t\rVert_{V_{t-1}^{-1} \tV_{t-1} V_{t-1}^{-1}} \lVert \theta_t - \bar\theta_t \rVert_{V_{t-1} \tV_{t-1}^{-1} V_{t-1}} + 2L \lVert \theta_t^{\star} - 
\bar\theta_t  \rVert_2 \quad (\forall a \in \mathcal{A}_t \lVert a \rVert_2 \leq L).
\end{align*}
As discussed in Section \ref{sec:analysis}, the two terms are upper bounded using different techniques. 
The first term is handled with the equivalent in a non-stationary environment of the deviation inequality of
Theorem \ref{theorem_deviation_weighted_sequential} and the second term is 
the equivalent of the bias.

\underline{Second step:} Upper bound for $\lVert \theta_t - \bar\theta_t \rVert_{V_{t-1} \tV_{t-1}^{-1} V_{t-1}}$. 

We have,
$$
\lVert \theta_t - \bar\theta_t \rVert_{V_{t-1} \tV_{t-1}^{-1} V_{t-1}}  \leq \lVert \theta_t - \hat{\theta}_{t-1}\rVert_{V_{t-1} \tV_{t-1}^{-1} V_{t-1}} + \lVert \bar\theta_t - 
\hat{\theta}_{t-1}\rVert_{V_{t-1} \tV_{t-1}^{-1} V_{t-1}} \leq 2 \beta_{t-1},
$$
where the last inequality holds because under our assumption $\bar\theta_t \in \mathcal{C}_t$ with high probability
and by definition $\theta_t \in \mathcal{C}_t$.
\newpage

\underline{Third step:} Upper bound for the bias. 

Let $D>0$,
\begin{align*}
\lVert \theta_t^{\star} - \bar\theta_t  \rVert_2 &= \lVert V_{t-1}^{-1}\sum_{s=1}^{t-1}\gamma^{-s} A_s A_s^{\top}
(\theta_s^{\star}- \theta_t^{\star}) \rVert_2 \\
 &\leq  \lVert \sum_{s=t-D}^{t-1} V_{t-1}^{-1} \gamma^{-s} A_s A_s^{\top}(\theta_s^{\star}- \theta_t^{\star}) \rVert_2 + 
  \lVert V_{t-1}^{-1}\sum_{s=1}^{t-D-1} \gamma^{-s} A_s A_s^{\top}(\theta_s^{\star}- \theta_t^{\star}) \rVert_2 \\
 &\leq  \lVert \sum_{s=t-D}^{t-1} V_{t-1}^{-1} \gamma^{-s} A_s A_s^{\top} \sum_{p=s}^{t-1}(\theta_p^{\star}- \theta_{p+1}^{\star}) \rVert_2 + 
 \lVert \sum_{s=1}^{t-D-1} \gamma^{-s} A_s A_s^{\top}(\theta_s^{\star}- \theta_t^{\star}) \rVert_{V_{t-1}^{-2}} \\
 & \leq  \lVert \sum_{p=t-D}^{t-1} V_{t-1}^{-1} \gamma^{-s} A_s A_s^{\top} \sum_{s=t-D}^{p}(\theta_p^{\star}- \theta_{p+1}^{\star}) \rVert_2 + 
 \frac{1}{\lambda}\sum_{s=1}^{t-D-1} \gamma^{t-1-s} \lVert A_s A_s^{\top}(\theta_s^{\star}- \theta_t^{\star}) \rVert_2 \\
  & \leq  \sum_{p=t-D}^{t-1}  \lVert V_{t-1}^{-1}\sum_{s=t-D}^{p} \gamma^{-s} A_s A_s^{\top} (\theta_p^{\star}- \theta_{p+1}^{\star}) \rVert_2 +
   \frac{2L^2S}{\lambda} \sum_{s=1}^{t-D-1} \gamma^{t-1-s} \\
  & \leq \sum_{p=t-D}^{t-1}  \lambda_{\max} \left(V_{t-1}^{-1} \sum_{s=t-D}^{p} \gamma^{-s} A_s A_s^{\top}\right)\lVert \theta_p^{\star} - 
  \theta_{p+1}^{\star}\rVert_2 + \frac{2L^2S}{\lambda} \frac{\gamma^{D}}{1-\gamma}.
\end{align*}
The first inequality is a consequence of the triangular inequality.
The third inequality uses that $V_{t-1}^{-2} \leq (\frac{\gamma^{t-1}}{\lambda})^2 I_d$.
In the last inequality, we have used the fact that for a symmetric matrix 
$M \in \mathcal{M}_d(\mathbb{R})$ 
and a vector $x \in \mathbb{R}^d$, 
$\lVert Mx\rVert_2\leq \lambda_{\max}(M) \lVert x \rVert_2$.

Furthermore,  for $x$ such that $\lVert x \rVert_2 \leq 1$, we have that 
for $t-D \leq p \leq t-1$,
\begin{align*}
x^{\top} V_{t-1}^{-1} \sum_{s=t-D}^{p} \gamma^{-s} A_s A_s^{\top} x &
\leq x^{\top} V_{t-1}^{-1} \sum_{s=1}^{t-1} \gamma^{-s} A_s A_s^{\top} x  +
 \lambda \gamma^{-(t-1)} x^{\top} V_{t-1}^{-1}x  \\
& \leq x^{\top} V_{t-1}^{-1} (\sum_{s=1}^{t-1} \gamma^{-s} A_s A_s^{\top} +
 \lambda \gamma^{-(t-1)} I_d)x = x^{\top}x \leq 1.
\end{align*}
Therefore, for all $p$ such that $ t-D \leq p \leq t-1,  \lambda_{\max} \left(V_{t-1}^{-1} \sum_{s=t-D}^{p} \gamma^{-s} A_s A_s^{\top}\right) \leq 1$.

By combining the second and the third step, with probability at least $1-\delta$:
$$
r_t \leq 2L\sum_{p=t-D}^{t-1}\lVert \theta_p^{\star}- \theta_{p+1}^{\star}\rVert_2 + \frac{4L^3S}{\lambda} \frac{\gamma^{D}}{1-\gamma} + 2\beta_{t-1} \lVert A_t\rVert_{V_{t-1}^{-1} \tV_{t-1} V_{t-1}^{-1}}.
$$

The assumption $\left| \langle A_t, \theta_t^{\star} \rangle \right|  \leq 1$ also implies $r_t \leq 2$. Hence, with probability at least $1-\delta$:
\begin{align}
r_t \leq 2L\sum_{p=t-D}^{t-1}\lVert \theta_p^{\star}- \theta_{p+1}^{\star}\rVert_2 + 4L^3S \frac{\gamma^{D}}{1-\gamma} + 2\beta_{t-1} \min(1,\lVert A_t\rVert_{V_{t-1}^{-1} \tV_{t-1} V_{t-1}^{-1}}).
\end{align}
To conclude the proof we use the results of Subsection \ref{sub-sec:elliptical-lemma}.

\underline{Final step:}
\begin{align*}
R_T &= \sum_{t=1}^T r_t \\
&\leq 2L \sum_{t=1}^T \sum_{p=t-D}^{t-1}\lVert \theta_p^{\star}- \theta_{p+1}^{\star}\rVert_2 + \frac{4L^3S}{\lambda} \frac{\gamma^{D}}{1-\gamma} T 
+ 2 \beta_{T} \sum_{t=1}^T  \min\left(1,\lVert A_t\rVert_{V_{t-1}^{-1} \tV_{t-1} V_{t-1}^{-1}}\right) \\
& \leq 2L \sum_{t=1}^T \sum_{p=t-D}^{t-1}\lVert \theta_p^{\star}- \theta_{p+1}^{\star}\rVert_2 + \frac{4L^3S}{\lambda} \frac{\gamma^{D}}{1-\gamma} T 
+ 2 \beta_{T} \sqrt{T} \sqrt{\sum_{t=1}^T  \min\left(1,\lVert A_t\rVert_{V_{t-1}^{-1} \tV_{t-1} V_{t-1}^{-1}}^2\right)}  \\
&\leq 2L B_T D + \frac{4L^3S}{\lambda} \frac{\gamma^{D}}{1-\gamma} T + 2\sqrt{2} \beta_T \sqrt{dT} \sqrt{T \log(1/\gamma) + 
 \log\left(1+ \frac{L^2}{d\lambda(1-\gamma)}\right)}.
\end{align*}
In the first inequality, we use that $t \mapsto \beta_t$ is increasing.
The second inequality is an application of the Cauchy-Schwarz inequality to 
the third term and the last inequality is an application of Corollary 
\ref{corollary:lemma-elliptic}.
\end{proof}

\subsection{Proof of Corollary \ref{corollary:asymptotic_regret_DLinUCB}}

\begin{proof}
Let $\gamma$ be defined as $\gamma = 1- (\frac{B_T}{dT})^{2/3}$ and $D= \frac{\log(T)}{(1-\gamma)}$.
With this choice of $\gamma$, $D$ is equivalent to $d^{2/3} B_T^{-2/3}T^{2/3} \log(T)$. Thus, $D B_T$
 is equivalent to $ d^{2/3} B_T^{1/3} T^{2/3} \log(T)$. 

In addition,
\begin{align*}
\gamma^D &= \exp(D \log(\gamma)) = \exp \left( \frac{\log(\gamma)}{1- \gamma} \log(T)\right)  \sim 1/T.
\end{align*}
Hence, $T \gamma^D \frac{1}{1-\gamma}$ behaves as $d^{2/3}T^{2/3} B_T^{-2/3} $.

Furthermore, $\log(1/\gamma) \sim d^{-2/3} B_T^{2/3} T^{-2/3}$, implying that $T \log(1/\gamma) \sim d^{-2/3} B_T^{2/3} T^{1/3}$.

As a result, it holds that, $\beta_T \sqrt{dT} \sqrt{T \log(1 / \gamma) + \log\left( 1 + \frac{L^2}{d \lambda (1- \gamma)}\right)}$ 
is equivalent to $ d T^{1/2} \sqrt{\log(T/B_T)} \sqrt{d^{-2/3} B_T^{2/3} T^{1/3}} = d^{2/3} B_T^{1/3} T^{2/3} \sqrt{\log(T/B_T)}$.

By adding those three terms and neglecting the log factors, we obtain the desired result.
\end{proof}

\section{A new analysis of the $\SWLinUCB$ algorithm}

In this section we propose a new analysis of the $\SWLinUCB$ algorithm. 
This is useful as the proof provided in \cite{cheung2018learning} has 
several gaps. First, Lemma 2 of \cite{cheung2018learning} is presented 
as a specific case of the analysis of \citep{abbasi2011improved}. It would 
hold in the case of a growing window, where the argument developed in 
\cite{abbasi2011improved} could be used, but not with a sliding window, 
where past actions are removed from the design matrix.
 Furthermore, Theorem 2 of \cite{cheung2018learning} that bounds 
 $\lvert \langle x, \hat{\theta}_{t-1} - \theta_t^{\star} \rangle \rvert$ for any fixed direction $x$ 
 with high probability is used in equation (42) 
  with $x$ replaced by $A_t$, whereas $A_t$ is a random variable strongly related to $\hat{\theta}_{t-1}$.
   
We only mention this analysis in the Appendix because the deviation 
inequalities established for the weighted model can not be used. 
Nevertheless, we believe that this analysis gives new insights on the problem with a
 sliding window.
 
 \subsection{Deviation inequality}
 
 Let us introduce some notations to clarify the model. We suppose 
 that there is a sliding window of length $\length$, such that the estimate 
 of the unknown parameter at time $t$ is based on the $\length$ last
  observations. The optimization program solved is
$$
\hat{\theta}_t = \argmin_{\theta \in \mathbb{R}^d} \left(\sum_{s= \max(1, t-\length +1)}^t 
(X_s - \langle A_s, \theta\rangle )^2 + \lambda/2 \lVert \theta \rVert_2^2)\right).
$$

One has
\begin{align}
\hat{\theta}_t = V_t^{-1} \sum_{s = \max(1, t-\length +1)}^t A_s X_s,
 \quad \text{where} \quad V_t =  \sum_{s = \max(1, t-\length +1)}^t A_s A_s^{\top} + \lambda I_d.
\end{align}
The expression linking the matrices $V_t$ and $V_{t-1}$ is the following
$$
V_t = V_{t-1} + A_t A_t^{\top } - A_{t- \length} A_{t- \length}^{\top}.
$$
 
The specificity of the sliding window model is that at time $t$, to update the 
design matrix, a new action vector $A_t$ is added but the oldest term $A_{t- \length}$
 is also removed . When considering the equivalent of the quantity
 $M_t(x)$ defined in the Appendix \ref{Appendix:confidence_bounds}, 
 the property of supermartingale does not hold anymore  
 because of this loss of information. For this reason, all the reasoning
 that was done in \citep{abbasi2011improved} can not be applied directly. 
 
The reward generation process we consider is still the one presented in Equation 
\ref{eq:reward_generation}. As for the $\DLinUCB$ model, the results are stated 
with $\sigma$-subgaussian random noises but the proofs are done with $\sigma = 1$.
Let $S_t =  \sum_{s = \max(1, t-\length +1)}^t A_s \eta_s$. 
 We start by giving the proof of the analogue of Lemma 2 
 presented in \cite{cheung2018learning}. We give an instantaneous 
 deviation inequality. 
 
\begin{proposition}[Instantaneous deviation inequality with a sliding window]
\label{prop:instantaneous_deviation_SW}
Let $t$ be a fixed time instant. For all $\delta > 0$,
$$
\mathbb{P}\left( \lVert S_t \rVert_{V_t^{-1}}\geq  \sigma\sqrt{ 2\log\left(\frac{1}{\delta}\right) +
 \log\left(\frac{\det(V_t)}{\lambda^d}\right)}   \right) \leq \delta.
$$ 
 \end{proposition}

\begin{proof}
We present an interesting trick in this proof for avoiding the loss 
of information due to the sliding window 
that is only usable for instantaneous deviation inequalities.

Let $t$ be the time instant of interest.
We assume that $t \geq \length$. We know that the estimate 
$\hat{\theta}_t$
 is only based on observations between time $t-\length +1$ to $t$. The trick 
 is to create a fictive regression model starting a time $t- \length$ 
 and receiving the exact same information as the true model 
 between the time instants $t-\length +1$ to $t$.
 
 To ease the understanding of the proof, the notations with 
dotted symbols refer to the fictive model.
 Let $u$ be a time instant in $[\![ t- \length, t]\!]$.
 Let $\dot{V}_u = \sum_{s=\max(1,t- \length +1)}^u A_s A_s^{\top} + \lambda I_d$, $\dot{S}_u =  
 \sum_{s=\max(1,t- \length+1)}^u A_s \eta_s$ and $\dot{M}_u(x) = \exp(x^{\top}\dot{S}_u - x^{\top} \dot{V}_u(0) x/2)$. Once again, $\dot{V}_u(0) = 
  \sum_{s= \max(1,t- \length +1)}^u A_s A_s^{\top}$ corresponds to the design matrix without the regularization term. 
By definition, $ \forall x \in \mathbb{R}^d, \dot{M}_{t-\length}(x) = 1$. 
 
 Using the $1$-subgaussianity and following the lines of the proof of Lemma \ref{lemma:exp_upper},  
 $$
  \mathbb{E}\lbrack \dot{M}_u(x) | \mathcal{F}_{u-1}\rbrack \leq \dot{M}_{u-1}(x).
  $$
  Therefore,   $ \forall u \in [\![ t- \length, t]\!], \mathbb{E}\lbrack \dot{M}_u(x)\rbrack 
  \leq   \mathbb{E}\lbrack \dot{M}_{t-\length}(x)\rbrack =1$.
 In particular for $u=t$, $\forall x \in \mathbb{R}^d, \mathbb{E}\lbrack 
 \dot{M}_t(x) \rbrack \leq 1$. By introducing a measure of
  probability $h = \mathcal{N}(0, \frac{1}{\lambda}I_d)$, we 
  still have $\mathbb{E}\left\lbrack\int \dot{M}_t(x)
  dh(x) \right\rbrack \leq 1$ using a similar reasoning than in 
  Lemma \ref{lemma:tilde}.
We can also give an exact formula 
for $\int \dot{M}_t(x) dh(x)$ with the chosen $h$. Let us remark that
 $\dot{S}_t = S_t$ and $\dot{V}_t = V_t$.
\begin{align*}
 \int_{\mathbb{R}^d}  \dot{M}_t(x)  dh(x)
    &= \frac{1}{\sqrt{(2 \pi)^d \det(1/\lambda I_d)}}\int_{\mathbb{R}^d} \exp\left( x^{\top}S_{t} -
    \frac{1}{2}\Vert x \Vert_{\lambda I_d}^2-\frac{1}{2}\Vert x \Vert_{V_{t}(0)}^2 \right) dx \\
    &= \frac{1}{\sqrt{(2 \pi)^d \det(1/\lambda I_d)}}\int_{\mathbb{R}^d}\exp\left(1/2\Vert S_{t}\Vert_{V_t^{-1}}^2 - 
    1/2\Vert x - V_t^{-1}S_{t}\Vert_{V_t}^2\right) dx \\
    &= \frac{\exp\left(\frac{1}{2}\Vert S_{t}\Vert_{V_t^{-1}}^2\right)}{\sqrt{(2 \pi)^d \det(1/\lambda I_d)}}\int_{\mathbb{R}^d}\exp\left( - 
    \frac{1}{2}\Vert x - V_t^{-1}S_{t}
    \Vert_{V_t}^2\right) dx \\
    &= \frac{\exp\left(\frac{1}{2}\Vert S_{t}\Vert_{V_t^{-1}}^2\right)}{\sqrt{(2 \pi)^d \det(1/\lambda I_d)}} \sqrt{(2\pi)^d \det\left(V_t^{-1}\right)} \\
    &= \exp\left(\frac{1}{2}\Vert S_{t}\Vert_{V_t^{-1}}^2\right) \sqrt{\frac{\det(\lambda I_d)}{\det(V_t)}}.
\end{align*}
For this reason,
\begin{align*}
    &
    \mathbb{P}\left(\lVert S_{t} \rVert_{V_{t}^{-1}} \geq \sqrt{2\log\left(\frac{1}{\delta}\right) + \log\left(\frac{\det(V_{t})}{\det(\lambda I_d)}\right)} \right) \\
    & \,\, 
    =  \mathbb{P}\left( \exp\left(\frac{1}{2}\Vert S_t\Vert_{V_{t}^{-1}}^2\right) \sqrt{\frac{\det(\lambda I_d)}{\det(V_{t})}} \geq \frac{1}{\delta}\right) \\
     & \,\,
      \leq  \delta \mathbb{E}\left\lbrack \int_{\mathbb{R}^d}  \dot{M}_t(x)  dh(x) \right\rbrack \,\, \textnormal{(Markov's inequality)} \\
       & \,\,
     \leq \delta.
\end{align*}
\end{proof}
The next step is to upper-bound the quantity $\det(V_t)$ similarly as in Proposition \ref{ineq_det_DLINUCB} for the weighted model.
 
\begin{proposition}[Determinant inequality for the design matrix with a sliding window]
In the specific case where $V_t$ is defined as $ V_t =  \sum_{s = \max(1, t-\length +1)}^t A_s A_s^{\top} + \lambda I_d$. 
Under the assumption $\forall t, \lVert A_t \rVert_2 \leq L$, the following holds,

$$
\det(V_t) \leq  \left( \lambda + \frac{L^2 \min(t, \length)}{d}\right)^d.
$$
\end{proposition}

The proof of this proposition is the same as in Proposition 
\ref{ineq_det_DLINUCB}.
By using the previous inequality, we can obtain the following proposition,

\begin{proposition} 
\label{prop:S_t_SW_deviation}
When using a sliding window model where the last $\length$ terms are considered, for all $\delta >0$,
$$
\mathbb{P}\left( \exists t \leq T, \lVert S_t \rVert_{V_t^{-1}}\geq  \sigma\sqrt{ 2\log\left(\frac{T}{\delta}\right) + d\log\left(
1 + \frac{L^2 \min(t, \length)}{\lambda d} 
\right)}   \right) \leq \delta.
$$
\end{proposition}

\begin{proof}
\begin{align*}
& \mathbb{P}\left( \exists t \leq T, \lVert S_t \rVert_{V_t^{-1}}\geq  \sigma\sqrt{ 2\log\left(\frac{T}{\delta}\right) + d\log\left(
1 + \frac{L^2 \min(t, \length)}{\lambda d} 
\right)}   \right) \\
& \quad
\leq \sum_{t=1}^T \mathbb{P}\left( \lVert S_t \rVert_{V_t^{-1}}\geq  \sigma\sqrt{ 2\log\left(\frac{T}{\delta}\right) + d\log\left(
1 + \frac{L^2 \min(t, \length)}{\lambda d} 
\right)}   \right) \\
& \quad
\leq \sum_{t=1}^T \mathbb{P}\left( \lVert S_t \rVert_{V_t^{-1}}\geq  \sigma\sqrt{ 2\log\left(\frac{T}{\delta}\right) + \log\left(\frac{\det(V_t)}{\lambda^d}\right)}   \right) \\
& \quad 
\leq \sum_{t=1}^T \frac{\delta}{T} \quad (\text{Proposition \ref{prop:instantaneous_deviation_SW}}) \leq \delta.
\end{align*}
\end{proof}

\subsection{Regret analysis}

The regret analysis of the $\SWLinUCB$ algorithm is  similar to the 
one proposed for $\DLinUCB$. We start by defining the confidence ellipsoid used
by the algorithm $\SWLinUCB$.

With the $\SWLinUCB$ algorithm, the $\beta_t$ term is defined in the following way,
\begin{align}
\label{beta_t_SW}
\beta_t = \sqrt{\lambda} S + \sigma\sqrt{ 2\log\left(\frac{T}{\delta}\right) + d\log\left(
1 + \frac{L^2 \min(t, \length)}{\lambda d} 
\right)} 
\end{align}
\underline{Remark:} The cost of loosing some information at each step due 
to the sliding window when $t> \length$ is the term $\log\left(\frac{T}{\delta} \right)$ 
rather than  $\log\left(\frac{1}{\delta} \right)$ in the definition of $\beta_t$.

Note that due to the use of a union bound technique the confidence 
radius is larger than the one suggested in \cite{cheung2018learning}. Nevertheless,
this was not taken into account in simulations for $\SWLinUCB$.

\begin{proposition}
\label{proposition:bar_theta_SW}
Let $\mathcal{C}_t =\bigg\{\theta \in \mathbb{R}^d: \lVert \theta - 
\hat{\theta}_{t-1} \rVert_{V_{t-1}^{-1}} \leq \beta_{t-1} \bigg\} $ denote the confidence ellipsoid.
Let $\bar\theta_t = V_{t-1}^{-1} \left(\sum_{s=\max(1, t-\length)}^{t-1} 
 A_s A_s^{\top} \theta_s^\star + \lambda \theta_t^\star \right)$. Then,
$\forall \delta >0$, 
 $$
\mathbb{P}\left( \forall t \geq 1, \bar\theta_t \in \mathcal{C}_t\right)\geq 1- \delta.
 $$
\end{proposition}

 \begin{proof}
\begin{align*}
 \bar \theta_{t} - \hat{\theta}_{t-1} &=V_{t-1}^{-1}\left( \sum_{s=\max(1, t-\length)}^{t-1}
  A_s A_s^{\top}\theta_{s}^{\star} + \lambda  \theta^{\star}_t - \sum_{s=\max(1, t-\length)}^{t-1}A_s A_s^{\top} \theta_{s}^{\star} - \sum_{s=\max(1, t-\length)}^{t-1} A_s \eta_s\right) \\
 &= - V_{t-1}^{-1} S_{t-1} + \lambda V_{t-1}^{-1} \theta_t^{\star}.
\end{align*}
Therefore,
\begin{align*}
\lVert \bar \theta_{t} - \hat{\theta}_{t-1}\rVert_{V_{t-1}^{-1}} &\leq \lVert S_{t-1} \rVert_{V_{t-1}^{-1}} + 
\lambda  \lVert \theta_t^{\star}\rVert_{V_{t-1}^{-1}} \\
&\leq \lVert S_{t-1} \rVert_{V_{t-1}^{-1}} + \sqrt{\lambda} S \quad
 (V_{t-1}^{-1} \leq \frac{1}{\lambda} I_d) \\
&\leq \beta_{t-1}  \quad \textnormal{(with probability $\geq 1-\delta$ thanks to Proposition \ref{prop:S_t_SW_deviation})}.
\end{align*}
\end{proof}
 We need to bound 
 the quantity $\sum_{t=1}^T \min\bigl(1, \lVert A_t \rVert_{V_{t-1}^{-1}}^2
 \bigr)$. An analysis of this quantity is already proved in \cite{cheung2018learning}.
  Nevertheless, we provide a simpler analysis in the following proposition.

\begin{proposition} 
\label{prop:elliptical_SW}
With the sliding window model, the following upper bound holds,
$$
\sum_{t=1}^T  \min\left(1,\lVert A_t\rVert_{V_{t-1}^{-1}}^2\right) \leq 2d \ceil{T/ \length} \log\left(  1+ \frac{\length L^2}{\lambda d} \right).
$$
\end{proposition}
\begin{proof}

We start by rewriting the sum as follows.
\begin{align*}
\sum_{t=1}^T   \min\left(1,\lVert A_t\rVert_{V_{t-1}^{-1}}^2\right) =
 \sum_{k=0}^{\ceil{T/\length}-1} \sum_{t= k \length + 1}^{(k+1)\length}  \min\left(1,\lVert A_t\rVert_{V_{t-1}^{-1}}^2\right) 
\end{align*}

For the $k$-th block of length $\length$ we define the matrix 
$W_t^{(k)} = \sum_{s= k \length +1}^t A_s A_s^{\top} + \lambda I_d$.
We also have $\forall t \in [\![ k\length, (k+1) \length]\!], V_t \geq W_t^{(k)}$ as every term in $W_t^{(k)}$ is 
contained in $V_t$ and the extra-terms in $V_t$ correspond
 to positive definite matrices. The matrices are definite positive, 
 thus $V_t^{-1} \leq (W_t^{(k)})^{-1}$ and consequently,
$$
\sum_{k=0}^{\ceil{T/\length}-1} \sum_{t= k \length + 1}^{(k+1)\length}
  \min\left(1,\lVert A_t\rVert_{V_{t-1}^{-1}}^2\right) 
   \leq \sum_{k=0}^{\ceil{T/\length}-1} \sum_{t= k \length + 1}^{(k+1)\length}  \min\left(1,\lVert A_t\rVert_{(W_{t-1}^{(k)})^{-1}}^2\right) 
$$

Furthermore, $\forall t \in[\![ k \length , (k+1) \length]\!] $
 we have,
$$ 
\det(W_t^{(k)}) = \det(W_{t-1}^{(k)}) \left( 1 +
 \lVert A_t \rVert_{(W_{t-1}^{(k)})^{-1}}^2 \right).
$$

With positive definitive matrices whose determinants 
are strictly positive,
 this implies that
$$
\frac{\det(W_{(k+1)\length}^{(k)})}{\det(W_{k \length}^{(k)})} = \prod_{t=k\length +1}^{(k+1) \length}
 \frac{\det(W_t^{(k)})}{\det(W_{t-1}^{(k)})} =  
 \prod_{t=k\length +1}^{(k+1) \length}  \left(1 +
  \lVert A_t\rVert_{ (W_{t-1}^{(k)})^{-1}}^2\right).
$$

By definition we have $W_{k \length}^{(k)} = \lambda I_d$ and  
$\forall x \geq 0, \min(1,x) \leq 2 \log(1+x)$. So,
\begin{align*}
\sum_{t=1}^T   \min\left(1,\lVert A_t\rVert_{V_{t-1}^{-1}}^2\right) &\leq 2  \sum_{k=0}^{\ceil{T/\length}-1} \sum_{t= k \length + 1}^{(k+1)\length}
 \log \left( 1 +\lVert A_t\rVert_{(W_{t-1}^{(k)})^{-1}}^2 \right) \\
& \leq 2 \sum_{k=0}^{\ceil{T/\length}-1} \log \left( \frac{\det(W_{(k+1)\length}^{(k)}   )}{\lambda^d} \right).
\end{align*}

Knowing that $W_{(k+1)\length}^{(k)}$ contains exactly $\length$
 terms allows us to give the following bound (by following the proof of Proposition \ref{ineq_det_DLINUCB}),
$$
\det(W_{(k+1)\length}^{(k)}) \leq \left(  \lambda + \frac{L^2 \length}{d}  \right)^d.
$$
Finally,
\begin{align*}
\sum_{t=1}^T   \min\left(1,\lVert A_t\rVert_{V_{t-1}^{-1}}^2\right) &\leq 2 d \ceil{T/ \length} \log\left( 1+ \frac{L^2 \length}{\lambda d} \right).
\end{align*}
\end{proof}

With those results we can give a high probability upper bound for the
 cumulative dynamic regret of the $\SWLinUCB$ 
algorithm.

\begin{theorem}
Assuming that $\sum_{s=1}^{T-1} \lVert \theta^{\star}_s - \theta^{\star}_{s+1}\rVert_2 \leq B_T$, the regret of the $\SWLinUCB$ 
algorithm may be bounded for all $\length >0$, with probability at least $1-\delta$, by
$$
R_T \leq 2L B_T \length + 2\sqrt{2} \beta_T \sqrt{dT} \sqrt{\ceil{T/ \length}} \sqrt{\log \left( 1 + \frac{L^2 \length}{\lambda d} \right)},
$$
where $\beta_T$ is defined in Equation (\ref{beta_t_SW}).
\end{theorem}
\begin{proof} 

\hspace{0.2cm}

\underline{1rst step:} Upper bound for the instantaneous regret 

Defining $A_t^{\star} = \argmax_{a \in \mathcal{A}_t} \langle a, \theta^{\star}_t\rangle$ and
$ \theta_t = \argmax_{\theta \in \mathcal{C}_t} \langle A_t, \theta \rangle.$
We have,
\begin{align*}
r_t &=  \max_{a \in \mathcal{A}_t} \langle a, \theta^{\star}_t \rangle - \langle A_t, \theta^{\star}_t \rangle  = \langle A_t^{\star} - A_t, 
\theta_t^{\star} \rangle \\
&=  \langle A_t^{\star} - A_t, \bar\theta_t \rangle + \langle A_t^{\star} - A_t, \theta_t^{\star} - \bar\theta_t \rangle 
\end{align*}

Under the event $\{ \forall t >0, \, \bar\theta_t \in \mathcal{C}_t\}$, 
that occurs with probability at least $1-\delta$ thanks to Proposition \ref{proposition:bar_theta_SW},
\begin{align}
\label{ineq:UCBs_SW}
\langle A_t^{\star}, \bar \theta_t \rangle \leq \argmax_{\theta \in \mathcal{C}_t} \langle A_t^{\star}, \theta \rangle = \text{UCB}_t(A_t^{\star}) 
\leq \text{UCB}_t(A_t) =  \argmax_{\theta \in \mathcal{C}_t} \langle A_t, \theta \rangle  = \langle A_t, \theta_t \rangle
\end{align}
Using Inequality (\ref{ineq:UCBs_SW}), with probability larger than $1-\delta$, $\forall t>0$,
\begin{align*}
r_t &\leq \langle A_t, \theta_t - \bar\theta_t \rangle +  \langle A_t^{\star} - A_t, \theta_t^{\star} - \bar\theta_t \rangle \\
&\leq \lVert A_t\rVert_{V_{t-1}^{-1}} \lVert \theta_t - \bar\theta_t \rVert_{V_{t-1}} + \lVert A_t^{\star} - A_t \rVert_2 \lVert \theta_t^{\star} - 
\bar\theta_t  \rVert_2 \quad \textnormal{(Cauchy-Schwarz)} \\
&\leq \lVert A_t\rVert_{V_{t-1}^{-1}} \lVert \theta_t - \bar\theta_t \rVert_{V_{t-1}} + 2L \lVert \theta_t^{\star} - \bar\theta_t  \rVert_2 
\quad (\text{Bounded action assumption}).
\end{align*}
As for the analysis of the regret for the $\DLinUCB$ algorithm, the two terms are upper bounded using different techniques. 
The first term is handled with  the deviation inequality of
Proposition \ref{proposition:bar_theta_SW}.

\underline{2nd step:} Upper bound for $\lVert \theta_t - \bar\theta_t \rVert_{V_{t-1}}$ 

We have,
$$
\lVert \theta_t - \bar\theta_t \rVert_{V_{t-1}}  \leq \lVert \theta_t - \hat{\theta}_{t-1}\rVert_{V_{t-1}} + \lVert \bar\theta_t - \hat{\theta}_{t-1}\rVert_{V_{t-1} } \leq 2 \beta_{t-1}.
$$
Where the last inequality holds because under our assumption $\bar\theta_t \in \mathcal{C}_t$ 
with probability at least $1-\delta$ and by definition $\theta_t \in \mathcal{C}_t$.

\underline{3rd step:} Upper bound for the bias. 

This step is similar to 
the proof proposed in \cite{cheung2018learning} for Lemma 1.
\begin{align*}
\lVert \theta_t^{\star} - \bar\theta_t  \rVert_2 &= 
\left\lVert V_{t-1}^{-1}\left(\sum_{s=\max(1,t- \length)}^{t-1}A_s A_s^{\top}
(\theta_s^{\star}- \theta_t^{\star}) \right) \right\rVert_2 \\
 &\leq  \left\lVert \sum_{s= \max(1,t- \length) }^{t-1} V_{t-1}^{-1}A_s A_s^{\top} 
 \sum_{p=s}^{t-1}(\theta_p^{\star}- \theta_{p+1}^{\star}) \right\rVert_2  \\
 & \leq \left\lVert \sum_{p=\max(1,t-\length)}^{t-1} V_{t-1}^{-1} 
 \sum_{s=\max(1,t-\length)}^{p} A_s A_s^{\top}  (\theta_p^{\star}- \theta_{p+1}^{\star}) \right\rVert_2 \\
  & \leq  \sum_{p=\max(1,t-\length)}^{t-1}  \left\lVert V_{t-1}^{-1}
  \sum_{s=\max(1,t-\length)}^{p} A_s A_s^{\top} 
  (\theta_p^{\star}- \theta_{p+1}^{\star}) \right\rVert_2 \\
  & \leq \sum_{p=\max(1,t-\length)}^{t-1}  \lambda_{\max} \left(V_{t-1}^{-1}
   \sum_{s=\max(1,t-\length)}^{p}  A_s A_s^{\top}\right)
   \lVert \theta_p^{\star}- \theta_{p+1}^{\star}\rVert_2.
\end{align*}
Furthermore,  for $x \in \mathbb{R}^d$ such that 
$\lVert x \rVert_2 \leq 1$, we have that for $\max(1,t-\length) \leq p \leq t-1$,
\begin{align*}
x^{\top} V_{t-1}^{-1} \sum_{s=\max(1,t-\length)}^{p} A_s A_s^{\top} x &
\leq x^{\top} V_{t-1}^{-1} \sum_{s=\max(1,t-\length)}^{t-1} A_s A_s^{\top} x  + 
\lambda x^{\top} V_{t-1}^{-1}x  \\
& \leq x^{\top} V_{t-1}^{-1} \left(\sum_{s=\max(1,t-\length)}^{t-1}  A_s A_s^{\top} +
 \lambda I_d\right)x = x^{\top}x \leq 1.
\end{align*}
By combining the second and the third step,
$$
r_t \leq 2L\sum_{p=\max(1,t-\length)}^{t-1}\lVert \theta_p^{\star}- \theta_{p+1}^{\star}\rVert_2 + 2\beta_{t-1} \lVert A_t\rVert_{V_{t-1}^{-1}}.
$$
By using the assumption $\forall a \in \mathcal{A}_t, \left| \langle A_t, \theta_t^{\star} \rangle \right|  \leq 1$, we also have $r_t \leq 2$.
So, with probability greater than $1-\delta$,
\begin{align}
r_t \leq 2L\sum_{p=\max(1,t-\length)}^{t-1}\lVert \theta_p^{\star}- \theta_{p+1}^{\star}\rVert_2 +
 2\beta_{t-1} \min\left(1,\lVert A_t\rVert_{V_{t-1}^{-1}}\right).
\end{align}
To conclude the proof, we use the results of Proposition \ref{prop:elliptical_SW}.

\underline{Final step:}
\begin{align*}
R_T &= \sum_{t=1}^T r_t \leq 2L \sum_{t=1}^T \sum_{p=\max(1,t-\length)}^{t-1}\lVert \theta_p^{\star}- \theta_{p+1}^{\star}\rVert_2 + 
2 \beta_{T} \sum_{t=1}^T  \min\left(1,\lVert A_t\rVert_{V_{t-1}^{-1}}\right) \\
& \leq 2L \sum_{t=1}^T \sum_{p=\max(1,t-\length)}^{t-1}\lVert \theta_p^{\star}- \theta_{p+1}^{\star}\rVert_2 + 
2 \beta_{T} \sqrt{T} \sqrt{\sum_{t=1}^T  \min\left(1,\lVert A_t\rVert_{V_{t-1}^{-1}}^2\right)}  \\
&\leq 2L B_T \length + 2\sqrt{2} \beta_T \sqrt{dT} \sqrt{\ceil{T/ \length}} \sqrt{ \log\left(  1+
 \frac{\length L^2}{\lambda d} \right)}.
\end{align*}
In the first inequality, we use the fact that $t \mapsto \beta_t$ is increasing.
The second inequality is an application of the Cauchy-Schwarz inequality to 
the second term. The last inequality is an application of Proposition \ref{prop:elliptical_SW}
\end{proof}

By denoting $\tilde{O}$ the function growth when omitting the logarithmic terms, we have the following Corollary.

\begin{corollary}[Asymptotic regret bound for $\SWLinUCB$]
If $B_T$ is known, by choosing $\length = (\frac{dT}{B_T})^{2/3}$, the regret of the $\SWLinUCB$ algorithm is 
asymptotically upper bounded with high 
probability by a term  $\tilde{O}(d^{2/3} B_T^{1/3}T^{2/3})$ when $T \to \infty$.

If $B_T$ is unknown, by choosing $\length = d^{2/3}T^{2/3}$, 
the regret of the $\SWLinUCB$ algorithm is asymptotically upper bounded with high probability by 
a term $\tilde{O}(d^{2/3} B_T T^{2/3})$ when $T \to \infty$.
\end{corollary}
 
 \begin{proof}
 With this particular choice of $\length$, we have:
 $$
 \length B_T \sim d^{2/3} T^{2/3} B_T^{1/3}.
 $$
 $\beta_T$ as defined by equation (\ref{beta_t_SW}) is equivalent to $\sqrt{d \log(T)}$.
 
 $\sqrt{T} \sqrt{\ceil{T/\length}}$ has a similar behavior than $d^{-1/3}T^{1-1/3} B_T^{1/3}$, 
 consequently the behavior of $ \beta_T \sqrt{dT} 
 \sqrt{\ceil{T/ \length}} \sqrt{ \log\left(  1+ \frac{\length L^2}{\lambda d} \right)}$ is similar 
 to $d^{2/3} B_T^{1/3} T^{2/3}\sqrt{\log(T)}\sqrt{\log(T/B_T)}$.
 
 By neglecting the logarithmic term, we have with high probability, 
$$ 
R_T = \tilde{O}_{T \to \infty}(d^{2/3} B_T^{1/3}T^{2/3}).
$$
 \end{proof}

\end{document}